\documentclass[11pt]{article}

\usepackage{jmlr2e}
\firstpageno{1}

\usepackage{color,xcolor}
\usepackage{microtype}
\usepackage{graphicx}
\usepackage{subfigure}
\usepackage{booktabs} 
\usepackage{multirow}
\usepackage{diagbox}
\usepackage{amssymb}
\usepackage{mathtools}
\usepackage{enumitem}
\usepackage{amsmath}

\usepackage{algorithm}
\usepackage{algorithmic}

\newtheorem{assumption}{Assumption}

\def \A {\mathbf{A}}

\def \R {\mathbb{R}}

\def \one {\mathbf{1}}
\def \ubf {\mathbf{u}}
\def \vbf {\mathbf{v}}

\def \x {\mathbf{x}}

\def \E {\mathbb{E}}

\def \a {\mathbf{a}}

\def \dbf {\mathbf{d}}
\def \z {\mathbf{z}}

\def \y {\mathbf{y}}

\def \Hbf {\mathbf{H}}
\def \g {\mathbf{g}}
\def \F {\mathcal{F}}

\def \C {\mathcal{C}}

\usepackage{hyperref}
\usepackage[capitalize,noabbrev]{cleveref}
\hypersetup{colorlinks={true},linkcolor={magenta},citecolor={blue}}

\begin{document}
\title{Online Learning for Non-monotone Submodular Maximization: From Full Information to Bandit Feedback} 

\author{\name Qixin Zhang \email qxzhang4-c@my.cityu.edu.hk\\
       \addr School of Data Science\\
       City University of Hong Kong\\
       Kowloon, Hong Kong, China
      \AND
       \name Zengde Deng \email zengde.dzd@cainiao.com \\
      \addr Cainiao Network\\
      Hang Zhou, China
       \AND
       \name Zaiyi Chen \email zaiyi.czy@cainiao.com \\
      \addr Cainiao Network\\
      Hang Zhou, China
      \AND
       \name Kuangqi Zhou\email
       kzhou@u.nus.edu\\
       \addr National University of Singapore\\
       117583, Singapore
      \AND
       \name Haoyuan Hu \email haoyuan.huhy@cainiao.com \\
      \addr Cainiao Network\\
      Hang Zhou, China
      \AND
      \name Yu Yang \email yuyang@cityu.edu.hk\\
      \addr School of Data Science\\
       City University of Hong Kong\\
       Kowloon, Hong Kong, China
      }
\maketitle

\begin{abstract}
In this paper, we revisit the online non-monotone continuous DR-submodular maximization problem over a down-closed convex set, which finds wide real-world applications in the domain of machine learning, economics, and operations research. At first, we present the \textbf{Meta-MFW} algorithm achieving a $1/e$-regret of $O(\sqrt{T})$ at the cost of $T^{3/2}$ stochastic gradient evaluations per round. As far as we know, \textbf{Meta-MFW} is the first algorithm to obtain $1/e$-regret of $O(\sqrt{T})$ for the online non-monotone continuous DR-submodular maximization problem over a down-closed convex set. Furthermore, in sharp contrast with \textbf{ODC} algorithm~\citep{thang2021online}, \textbf{Meta-MFW} relies on the simple online linear oracle without discretization, lifting, or rounding operations. Considering the practical restrictions, we then propose the \textbf{Mono-MFW} algorithm, which reduces the per-function stochastic gradient evaluations from $T^{3/2}$ to 1 and achieves a $1/e$-regret bound of $O(T^{4/5})$. Next, we extend \textbf{Mono-MFW} to the bandit setting and propose the \textbf{Bandit-MFW} algorithm which attains a $1/e$-regret bound of $O(T^{8/9})$. To the best of our knowledge, \textbf{Mono-MFW} and \textbf{Bandit-MFW} are the first sublinear-regret algorithms to explore the one-shot and bandit setting for online non-monotone continuous DR-submodular maximization problem over a down-closed convex set, respectively. Finally, we conduct numerical experiments on both synthetic and real-world datasets to verify the effectiveness of our methods.
\end{abstract}

\section{Introduction}
Continuous DR-submodular maximization draws wide attention since it mathematically depicts the diminishing return phenomenon in continuous domains. Numerous real-world applications in machine learning, operations research, and other related areas, such as non-definite quadratic programming~\citep{ito2016large},
revenue maximization~\citep{soma2017non,bian2020continuous}, viral marketing~\citep{kempe2003maximizing,yang2016continuous}, determinantal point processes~\citep{kulesza2012determinantal,mitra2021submodular+}, to name a few, could be modeled throughout the notion of continuous DR-submodularity. 

In recent years, the prominent paradigm of online optimization~\citep{zinkevich2003online,hazan2016introduction} has led to spectacular successes in modelling the imperfect and complicated environment. 
In this framework, at each step, the online algorithm first chooses an action from a predefined set of feasible actions; then the adversary reveals the utility function. 
The objective of the online algorithm is to minimize the gap between the accumulative reward and that of the best fixed policy in hindsight. 

Previously, a large body of algorithms~\citep{bian2020continuous,mokhtari2020stochastic} with approximation guarantees rely on the monotone assumption of continuous DR-submodular functions. However, many real-world problems, such as the general DR-submodular quadratic programming~\citep{ito2016large} and revenue maximization~\citep{soma2017non}, are instances of non-monotone DR-submodular maximization. Motivated by these real applications, in this paper we focus on the problem of online non-monotone continuous DR-submodular maximization over a down-closed convex set under different feedbacks, i.e., full-information and bandit feedback.

Recently, based on a special online non-convex oracle, ~\citet{thang2021online} presented the first online algorithm~(\textbf{ODC}) for non-monotone continuous DR-submodular maximization over a down-closed convex set. \textbf{ODC} achieves a $1/e$-regret of $O(T^{3/4})$ where $T$ is the horizon. Notably, the non-convex oracles of \textbf{ODC} need to discretize the original constrained domain and lift the $n$-dimensional subroutine problem into a solvable linear programming in a higher $(M\times n)$-dimensional space where $M=(T/n)^{1/4}$, which will incur a heavy computation burden when $T$ is large. Moreover, the rounding operation~\citep{mirrokni2017tight} in the online non-convex oracle assumes the knowledge of the vertices of the down-closed convex set, which is infeasible in many real applications. 
In this paper, we propose a new method to overcome these drawbacks.
Motivated via the measured continuous greedy \citep{feldman2011unified}, we first present the Meta-Measured Frank-Wolfe (\textbf{Meta-MFW}) algorithm, which achieves a faster $1/e$-regret of $O(T^{1/2})$ with only simple online linear oracle.

Note that \textbf{ODC} and \textbf{Meta-MFW} require inquiring $T^{3/4}$ and $T^{3/2}$ stochastic gradient evaluations at each round, respectively.
Therefore, when $T$ is large, the huge amount of gradient estimates at each round makes both algorithms computationally prohibitive.
In many scenarios, the stochastic gradient is 1)~time-consuming to acquire, for instance, in the influence maximization task~\citep{kempe2003maximizing,yang2016continuous}, we need to generate enormous samples on large-scale social graphs to estimate the gradient, or 2)~impossible to compute, e.g., black-box attacks and optimization~\citep{chen_py2017zoo,ilyas2018black,chen2020black}. 
Considering these practical limitations, we also want to extend our proposed \textbf{Meta-MFW} into both one-shot and bandit feedback scenarios. 
As for the one-shot setting, we merge the blocking procedures~\citep{zhang2019online} into \textbf{Meta-MFW} to present \textbf{Mono-MFW} algorithm which yields a result with a $1/e$-regret of $O(T^{4/5})$ and reduces the number of per-function stochastic gradient evaluations from $T^{3/2}$ (or $T^{3/4}$) to $1$. 
Finally, in the bandit feedback where the only observable information is the reward we receive, a new algorithm \textbf{Bandit-MFW} is proposed with the exploration-exploitation policy~\citep{zhang2019online} and achieves $1/e$-regret of $O(T^{8/9})$.

To be specific, we make the following contributions:
\begin{enumerate} 
	\item We first develop a new algorithm \textbf{Meta-MFW} for online non-monotone continuous DR-submodular maximization problem over a down-closed convex set, which only relies on the simple online linear oracle without discretization, lifting, or rounding operations. 
	Moreover, in sharp contrast with \textbf{ODC}~\citep{thang2021online}, \textbf{Meta-MFW} achieves a faster $1/e$-regret of $O(T^{1/2})$ at the cost of $T^{3/2}$ stochastic gradient evaluations for each reward function. 
	It is worth mentioning that the $1/e$-regret of $O(T^{1/2})$ result not only achieves the best-known approximation guarantee for the offline problem~\citep{bian2017continuous} but also matches the optimal $O(\sqrt{T})$ regret~\citep{hazan2016introduction}. Meanwhile, like \textbf{ODC} algorithm, our proposed \textbf{Meta-MFW} also can achieve $1/e$-regret of $O(T^{3/4})$ with $T^{3/4}$ per-function stochastic gradient evaluations.
	
	\item Considering the practical restrictions, we then present the one-shot algorithm \textbf{Mono-MFW} equipped with blocking procedures~\citep{zhang2019online}, which achieves a $1/e$-regret of $O(T^{4/5})$ and reduces the stochastic gradient evaluations from $T^{3/2}$ or $T^{3/4}$ to $1$ at each round. Next, in the bandit setting, we propose the \textbf{Bandit-MFW} algorithm achieving a $1/e$-regret of $O(T^{8/9})$ by only inquiring one-point function value for each reward function. To the best of our knowledge,  \textbf{Mono-MFW} and \textbf{Bandit-MFW} are the first sublinear-regret algorithm to explore the one-shot and bandit settings for online non-monotone continuous DR-submodular maximization problem over a down-closed convex set, respectively.
	
	\item Finally, we empirically evaluate our proposed methods on both synthetic and real-world datasets. Numerical experiments demonstrate the superior performance of our proposed algorithms.
\end{enumerate}

\subsection{Related Work}
Continuous DR-submodular maximization problem has been extensively investigated as it admits efficient approximate maximization routines. In this section, we provide a summary about these known results.

\noindent{\textbf{Monotone Setting:}}
In the deterministic setting, \citet{bian2017guaranteed} first proposed a variant of Frank-Wolfe achieving   $(1-1/e)OPT-\epsilon$ after $O(1/\epsilon)$ iterations where $OPT$ is the optimal objective value. When a stochastic gradient oracle is available, \citet{hassani2017gradient} proved that the stochastic gradient ascent guarantees $(1/2)OPT-\epsilon$ after $O(1/\epsilon^{2})$ iterations. Then, \citet{mokhtari2018conditional} proposed the stochastic continuous greedy algorithm, which achieves a $(1-1/e)$-approximation after $O(1/\epsilon^{3})$ iterations.  After that,  an accelerated stochastic continuous greedy algorithm is presented in \citet{hassani2020stochastic}, which guarantees a $(1-1/e)OPT-\epsilon$ after $O(1/\epsilon^{2})$ iterations. As for the online settings, \citet{chen2018online} first investigated the online gradient ascent with a $(1/2)$-regret of $O(\sqrt{T})$. Then, inspired by the meta actions \citep{streeter2008online}, \citet{chen2018online} proposed the Meta-Frank-Wolfe algorithm with a $(1-1/e)$-regret bound of $O(\sqrt{T})$ under the deterministic setting. With an unbiased gradient oracle,
then \citet{chen2018projection} proposed a variant of the Meta-Frank-Wolfe algorithm having a $(1-1/e)$-regret bound of $O(T^{1/2})$ and requiring $T^{3/2}$ stochastic gradient queries for each function. In order to reduce the number of gradient evaluations, \citet{zhang2019online} presented Mono-Frank-Wolfe taking the blocking procedure, which achieves a $(1-1/e)$-regret bound of $O(T^{4/5})$ with only one stochastic gradient evaluation at each round. Leveraging this one-shot algorithm, \citet{zhang2019online} presented a bandit algorithm Bandit-Frank-Wolfe achieving  $(1-1/e)$-regret bound of $O(T^{8/9})$. Recently, based on a novel auxiliary function, \citet{zhang2022boosting} have presented a variant of gradient ascent improving the approximation ratio of the standard gradient ascent \citep{hassani2017gradient,chen2018online} from $1/2$ to $1-1/e$ in both offline and online settings.

\begin{table}[t]
	\renewcommand\arraystretch{1.35}
	\centering
	\caption{Comparison of regrets for online non-monotone continuous DR-submodular function maximization over a down-closed convex set with stochastic gradient oracles. '\textbf{\# Grad}' means the number of stochastic gradient evaluations at each round; '\textbf{Oracle}' indicates which type of online oracle used in algorithms; '\textbf{Ratio}' means the approximation ratio, and '\textbf{Feedback}' indicates full-information or bandit feedback scenario in the online learning.}
	\resizebox{\textwidth}{!}{
    \begin{tabular}{c|c|c|c|c|c}
    \toprule[1.5pt]
        \textbf{Method} &\textbf{Ratio}& \textbf{Regret}& \textbf{\# Grad}&\textbf{Oracle}&\textbf{Feedback} \\
        \hline
        \textbf{ODC}~\citep{thang2021online}&$1/e$&$O(T^{3/4})$&$T^{3/4}$&non-convex&full\\
        \hline 
        \textbf{Meta-MFW}~(This paper)&$1/e$&$O(T^{1/2})$&$T^{3/2}$&linear&full\\
        \hline 
        \textbf{Meta-MFW}~(This paper)&$1/e$&$O(T^{3/4})$&$T^{3/4}$&linear&full\\
        \hline 
        \textbf{Mono-MFW}~(This paper)&$1/e$&$O(T^{4/5})$&$1$&linear&full\\
        \hline
        \textbf{Bandit-MFW}~(This paper)&$1/e$&$O(T^{8/9})$&$0$&linear&bandit\\
        \midrule[1.5pt]
    \end{tabular}
	}
	\label{tab:online_convergence}
\end{table}

\noindent{\textbf{Non-Monotone Setting:}} Without the monotone property, maximizing the continuous DR-submodular function becomes much harder. Under the down-closed convex
constraint, \citet{bian2017continuous} proposed the deterministic Two-Phase Frank-Wolfe and Non-monotone Frank-Wolfe with $1/4$-approximation and $1/e$-approximation guarantee, respectively. When only an unbiased estimate of gradient is available, \citet{hassani2020stochastic} improved the Non-monotone Frank-Wolfe by variance reduction technique, which yields a result with $1/e$-approximation guarantee after $O(1/\epsilon^{3})$ iterations. Moreover, inspired by the Double Greedy \citep{buchbinder2015tight,buchbinder2018deterministic} for discrete unconstrained submodular set maximization, \citet{niazadeh2018optimal} and \citet{bian2019optimal} proposed a similar $1/2$-approximation algorithms for unconstrained continuous DR-submodular maximization. Note that \citet{vondrak2013symmetry} pointed that any algorithm with a constant-factor approximation for maximizing a non-monotone DR-submodular function over a non-down-closed convex set would require exponentially many value queries and the approximation guarantee of $1/2$ is tight for unconstrained DR-submodular maximization. \citet{thang2021online} is the first work to explore the sublinear-regret online algorithm for the non-monotone continuous DR-submodular maximization problems over a down-closed convex set.

We present a comparison between this work and previous studies in Table~\ref{tab:online_convergence}.

\section{Preliminaries}
\noindent{\textbf{Notation:}} In this paper, a lower boldface $\x$ denotes a vector with suitable dimension and an uppercase boldface $\A$ for a matrix. For each vector $\x$, the $i$-th element of $\x$ is denoted as $(\x)_{i}$. Specially, $\mathbf{0}$ and $\one$ represent the vector whose elements are all zero or one, respectively. For any positive integer number $K$, the symbol $[K]$ denotes the set $\{1,\dots, K\}$. Moreover, the symbol $\odot$ and $\oslash$ denote coordinate-wise multiplication and coordinate-wise division, respectively. For instance, given two vector $\x$ and $\y$, if $\y>\mathbf{0}$, the $i$-th element of vector $\x\oslash\y$ is $\frac{(\x)_{i}}{(\y)_{i}}$. The product $\langle\x,\y\rangle=\sum_{i}(\x)_{i}(\y)_{i}$ and the norm $\left\|\x\right\|=\sqrt{\langle\x,\x\rangle}$. We say the domain $\C\subseteq[0,1]^{n}$ is down-closed, if there exist a lower vector $\underline{\ubf}\in\C$ such that 1) $\y\ge\underline{\ubf}$ for any $\y\in\C$; 2) $\x\in\C$ if there exists a vector $\y\in\C$ satisfying $\underline{\ubf}\le\x\le\y$. Additionally, the radius $r(\mathcal{C})=\max_{\x\in\C} \left\|\x\right\|$ and the diameter $\mathrm{diam}(\C)=\max_{\x,\y\in\C} \left\|\x-\y\right\|$.

\noindent{\textbf{DR-Submodularity:}} 
A differentiable function $f:[0,1]^{n}\rightarrow\R_{+}$ is DR-submodular iff $\nabla f(\x)\le\nabla f(\y)$ when $\x\ge\y$~\citep{bian2020continuous}.

\noindent{\textbf{Smoothness:}} A differentiable function $f$ is called $L_{0}$-$smooth$ if for any $\x,\y\in[0,1]^{n}$, $\left\|\nabla f(\x)-\nabla f(\y)\right\|\le L_{0}\left\|\x-\y\right\|$.

\noindent{\textbf{Problem Settings and $\alpha$-regret:}} In this paper, we revisit the online non-monotone continuous non-monotone DR-submodular maximization problem over a down-closed convex set $\C$. For a $T$-$round$ game, after the learner chooses an action $\x_{t}\in\C$ at each round, the adversary reveals a DR-submodular function $f_{t}:[0,1]^{n}\rightarrow\R_{+}$ and feeds back the reward $f_{t}(\x_{t})$ to the learner. The 
goal is to design efficient algorithms such that the gap between the accumulative reward and that of the best fixed policy in hindsight with scale parameter $\alpha$, i.e., $\mathcal{R}_{\alpha}(T)=\alpha\max_{\x\in\C}\sum_{t=1}^{T}f_{t}(\x)-\sum_{t=1}^{T}f_{t}(\x_{t})$, is sublinear in horizon $T$. That is, $\lim_{T\rightarrow\infty}\mathcal{R}_{\alpha}(T)/T=0$. In this paper, we consider $\alpha=1/e$.

\section{Algorithms and Main Results}
\subsection{Online Non-monotone Continuous DR-submodular maximization}\label{sec:Meta}
In this subsection, we present a new online algorithm (Algorithm~\ref{alg:1}) for non-monotone continuous DR-submodular maximization over a down-closed convex set, which is inspired by the measured continuous greedy~\citep{feldman2011unified,mitra2021submodular+} and the meta-action framework~\citep{streeter2008online,chen2018projection} which utilizes the online linear optimization oracles~\citep{hazan2016introduction}. Note that an online linear optimization oracle is an instance of the off-the-shelf online linear maximization algorithm that sequentially maximizes linear objectives.

In sharp contrast with the Meta-Frank-Wolfe~\citep{chen2018projection} for online monotone continuous DR-submodular maximization, in our Algorithm~\ref{alg:1} we adopt a different update rule (line 6) and a novel feedback (line 11).
Given a series of update directions $\vbf_{t}^{(k)}\in\C,\forall k\in[K]$ and initial point $\x_{t}^{(0)}=\mathbf{0}$, we consider 
\begin{equation}\label{equ:update}
	\x_{t}^{(k)}=\x_{t}^{(k-1)}+\frac{1}{K}\vbf_{t}^{(k)}\odot(\one-\x_{t}^{(k-1)}),
\end{equation} 
where we re-weight the $i$-th element of $\vbf_{t}^{(k)}$ by $(\one-\x_{t}^{(k-1)})_{i}$ at each round and push the iteration point $\x_{t}^{(k-1)}$ along the weighted update direction $\vbf_{t}^{(k)}\odot(\one-\x_{t}^{(k-1)})$ with step size $\frac{1}{K}$. 
Due to the update rule of Equation~\eqref{equ:update}, then Algorithm~\ref{alg:1} feeds back the weighted gradient estimate $\g_{t}^{(k)}\odot(\one-\x_{t}^{(k)})$ for the linear oracle $\mathcal{E}^{(k)}$, where we view the vector $\g_{t}^{(k)}$ as an estimate for $\nabla f_t(\x_{t}^{(k)})$. Our update rule guarantees that $\x_{t}^{(k)} \in \C$ (proof in Appendix B).

Next, we demonstrate how the $K$ different linear oracles work. Each linear oracle $\mathcal{E}^{(k)}$ in Algorithm~\ref{alg:1} tries to online maximize the cumulative linear reward function $\sum_{t=1}^{T}\langle\g_{t}^{(k)}\odot(\one-\x_{t}^{(k)}),\cdot\rangle$. Precisely, after $\mathcal{E}^{(k)}$ commits to the action $\vbf_{t}^{(k)}\in\C$ at $t$-th round, Algorithm~\ref{alg:1} feeds back the vector $\g_{t}^{(k)}\odot(\one-\x_{t}^{(k)})$ and the reward $\langle\g_{t}^{(k)}\odot(\one-\x_{t}^{(k)}),\vbf_{t}^{(k)}\rangle$ to the oracle $\mathcal{E}^{(k)}$; then the oracle $\mathcal{E}^{(k)}$ updates the action via some well-known strategies such as the gradient descent or regularized-follow-the-leader~\citep{hazan2016introduction}. Taking the gradient descent as an example, the oracle $\mathcal{E}^{(k)}$ will choose the next action $\vbf_{t+1}^{(k)}=\arg\min_{\vbf\in\C}\|\vbf-(\vbf_{t}^{(k)}+\frac{1}{\sqrt{T}}\g_{t}^{(k)}\odot(\one-\x_{t}^{(k)}))\|$. Predictably, compared with the complicated online non-convex oracle of \textbf{ODC} (See online vee learning algorithm in \citep{thang2021online}), the online linear oracle in the \textbf{Meta-MFW}, without discretization, lifting, or rounding operations, is simpler and more efficient.
\begin{algorithm}[t]
	\caption{Meta-Measured Frank-Wolfe~(\textbf{Meta-MFW})}
	\begin{algorithmic}[1]\label{alg:1}
		\STATE {\bf Input:} $K$ online linear maximization oracles over $\C$, i.e, $\mathcal{E}^{(1)},\dots,\mathcal{E}^{(K)}$, $\eta_{k}$, $\g_{t}^{(0)}=\x_{t}^{(0)}=\mathbf{0}$.
		\STATE {\bf Output:} $\y_1,\dots,\y_T$.
		\FOR{$t=1,\ldots, T$}
		\FOR{$k=1,\dots,K$}
		\STATE Receive $\vbf_{t}^{(k)}$ which is the output of oracle $\mathcal{E}^{(k)}$.
		\STATE $\x_{t}^{(k)}=\x_{t}^{(k-1)}+\frac{1}{K}\vbf_{t}^{(k)}\odot(\one-\x_{t}^{(k-1)})$.
		\ENDFOR
		\STATE Play $\y_{t} = \x_{t}^{(K)}$ for $f_{t}$ to get reward $f_{t}(\y_{t})$ and observe the stochastic gradient information of $f_{t}$.
		\FOR{$k=1,\dots,K$}
		\STATE $\g_{t}^{(k)}=(1-\eta_{k})\g_{t}^{(k-1)}+\eta_{k}\widetilde{\nabla}f_{t}(\x_{t}^{(k)})$ where $\E(\widetilde{\nabla}f_{t}(\x_{t}^{(k)})|\x_{t}^{(k)})=\nabla f_{t}(\x_{t}^{(k)})$.
		\STATE Feed back $\langle\g_{t}^{(k)}\odot(\one-\x_{t}^{(k)}),\vbf_{t}^{(k)}\rangle$ as the payoff to be received by oracle $\mathcal{E}^{(k)}$.
		\ENDFOR
		\ENDFOR
	\end{algorithmic}
\end{algorithm}

We then make some assumptions for the regret analysis of Algorithm~\ref{alg:1}.
\begin{assumption}\label{assumption1} \
	\begin{enumerate}
		\item[(i)]The domain $\C\subseteq[0,1]^{n}$ is a down-closed convex set including the original point $\mathbf{0}$, where $n$ is the dimensional parameter.
		\item[(ii)] Each $f_{t}:[0,1]^{n}\rightarrow \mathbb{R}_{+}$ is a differentiable, DR-submodular function with smoothness parameter $L_{0}$.
		\item[(iii)] For any linear maximization oracle $\mathcal{E}^{(k)}$, the regret at horizon $t$ is at most $M_{0}\sqrt{t}$, where $M_{0}$ is a parameter.
	\end{enumerate}
\end{assumption} 
\begin{assumption}\label{add_assumption}
	For any $t\in[T]$ and $\x\in[0,1]^{n}$, there exists a stochastic gradient oracle $\widetilde{\nabla}f_{t}(\x)$ with  $\E(\widetilde{\nabla}f_{t}(\x)|\x)=\nabla f_{t}(\x)$ and $\E(\|\nabla f_{t}(\x)-\widetilde{\nabla}f_{t}(\x)\|^{2})\le\sigma^{2}$.
\end{assumption}
 
\begin{theorem}[Proof in \cref{appendix:Meta}]\label{thm:1}
Under Assumption~\ref{assumption1} and ~\ref{add_assumption}, if we set $\eta_{k}=\frac{2}{(k+3)^{2/3}}$, we could verify that Algorithm~\ref{alg:1} achieves:
	\begin{equation*}
		\begin{split}
		\frac{1}{e}\sum_{t=1}^{T}f_{t}(\x^{*})-\sum_{t=1}^{T}\E(f_{t}(\y_{t}))
		\le M_{0}\sqrt{T}+L_{0}r^{2}(\C)\frac{T}{2K}+\frac{r(\C)}{2}(3N_{0}+1)\frac{T}{K^{1/3}},
		\end{split}
	\end{equation*}
	where $N_{0}=\max\{4^{2/3}\max_{t\in[T]}\|\nabla f_{t}(\x_{t}^{(1)})\|^{2},4\sigma^{2}+6(L_{0}r(\C))^{2}\}$ and $\x^{*}=\arg\max_{\x\in\C}\sum_{t=1}^{T}f_{t}(\x)$.
\end{theorem}
\begin{remark}
	According to Theorem~\ref{thm:1}, if we set $K=T^{3/2}$, \textbf{Meta-MFW} yields the first result to achieve a $1/e$-regret of $O(\sqrt{T})$, which is faster than the previous outcome of \textbf{ODC}~\citep{thang2021online}. It is worth mentioning that the $1/e$-regret of $O(\sqrt{T})$ not only achieves the best-known guarantee for the offline problem, but also matches the optimal $O(\sqrt{T})$ regret of online convex optimization~\citep{hazan2016introduction}. 
\end{remark}
\begin{remark}
	Meanwhile, when $K=T^{3/4}$, \textbf{Meta-MFW} achieves a $1/e$-regret of $O(T^{3/4})$, which has the same approximation ratio and regret as \textbf{ODC}~\citep{thang2021online}. Although the oracle number $K=T^{3/4}$ of \textbf{Meta-MFW} is the same as \textbf{ODC}, \textbf{Meta-MFW} is more time-efficient than \textbf{ODC} since we adopt the simple online linear oracles while \textbf{ODC} utilizes complicated online non-convex oracles with discretization, lifting, and rounding operations.
\end{remark}

\subsection{One-shot Online Non-monotone Continuous DR-submodular Maximization}\label{sec:Mono}
In many real-world scenarios, it could be time-consuming or even impossible to compute the stochastic gradient, e.g., influence maximization \citep{yang2016continuous} as well as black-box attacks \citep{ito2016large}.
Thus, our Algorithm~\ref{alg:1}, which needs to inquire $K$ gradient estimates for each reward function $f_{t}$ (line 10 in \textbf{Meta-MFW}), seems to be restrictive for many applications. 
To tackle the practical challenges, we hope to extend our proposed \textbf{Meta-MFW} into one-shot or bandit settings, where we only are permitted to inquire an unbiased gradient or one-point function value for each $f_{t}$, respectively. At first, we investigate the one-shot non-monotone DR-submodular maximization in this subsection. 

We begin by reviewing the fairly known blocking technique in online learning~\citep{hazan2016introduction,zhang2019online}. 
Specifically, we divide the $T$ reward functions $f_{1},\dots,f_{T}$ into $Q$ blocks of the same size $K$, where $T=QK$, i.e., the $q$-th block includes the $K$ different functions $f_{(q-1)K+1},\dots,f_{qK}$. 
We also define the average function in the $q$-th block as $\bar{f}_{q}=\sum_{t=(q-1)K+1}^{qK}f_{t}/K$. 
To reduce the number of per-function stochastic gradient evaluations, the key idea is to view each $\bar{f}_{q}$ as a virtual reward function, such that the original $T$-round online optimization can be transferred into a new $Q$-round game. In this new $Q$-round game, at the $q$-th step, the algorithm first chooses an action $\x_{q}\in\C$, then the adversary reveals the reward $\bar{f}_{q}(\x_{q})$ for the algorithm. 

Since each $\bar{f}_{q}$ is also continuous DR-submodular, we could directly adopt Algorithm~\ref{alg:1} to tackle the new $Q$-round game, which also requires inquiring $K$ unbiased gradient estimates for each $\bar{f}_{q}$. Note that, in $q$-th block, there exist $K$ different stochastic gradient oracles $\{\widetilde{\nabla}f_{(q-1)K+1},\dots,\widetilde{\nabla}f_{qK}\}$. Moreover, for each random permutation $\{t_{q}^{(1)},\dots,t_{q}^{(K)}\}$ of the indices $\{(q-1)K+1,\dots,qK\}$, it could be verified that the $\E(f_{t_{q}^{(k)}}(\x)|\x)=\bar{f}_{q}(\x)$ and $\E(\widetilde{\nabla}f_{t_{q}^{(k)}}(\x)|\x)=\nabla\bar{f}_{q}(\x)$.
As a result, we can construct unbiased gradient estimates of $\bar{f}_{q}$ at $K$ different points via the $K$ existing oracles $\{\widetilde{\nabla}f_{(q-1)K+1},\dots,\widetilde{\nabla}f_{qK}\}$, and each oracle inquires only one gradient evaluation.
In this manner, we successfully reduce the number of per-function gradient evaluations from $K$ to $1$.
Motivated via this high-level idea, we present a one-shot variant in Algorithm~\ref{alg:2}~(\textbf{Mono-MFW}). Note that in the $q$-th block, we play the same point $\y_{t}=\x_{q}^{(K)}$ for each objective function in $\{f_{(q-1)K+1},\dots,f_{qK}\}$.
We provide the regret analysis of Algorithm~\ref{alg:2} in Theorem~\ref{thm:2}.
\begin{algorithm}[t]
	\caption{Mono-Measured Frank-Wolfe~(\textbf{Mono-MFW})}
	\begin{algorithmic}[1]\label{alg:2}
		\STATE{\bf Input:} $K$ online linear maximization oracles over $\C$, i.e., $\mathcal{E}^{(1)},\dots,\mathcal{E}^{(K)}$, $Q=\frac{T}{K}$, $\eta_{k}$, $\g_{q}^{(0)}=\x_{q}^{(0)}=\mathbf{0}$.
		\STATE {\bf Output:} $\y_1,\dots,\y_T$.
		\FOR{$q=1,\ldots, Q$}
		\FOR{$k=1,\dots,K$}
		\STATE Receive the update direction $\vbf_{q}^{(k)}$ which is the output of oracle $\mathcal{E}^{(k)}$.
		\STATE $\x_{q}^{(k)}=\x_{q}^{(k-1)}+\frac{1}{K}\vbf_{q}^{(k)}\odot(\one-\x_{q}^{(k-1)})$.
		\ENDFOR
		\STATE Generate a random permutation $\{t_{q}^{(1)},\dots,t_{q}^{(K)}\}$ for $\{(q-1)K+1,\dots,qK\}$.
		\FOR{$t=(q-1)K+1,\dots,qK$} 
		\STATE Play $\y_{t}=\x_{q}^{(K)}$ to get reward $f_{t}(\y_{t})$ and observe the stochastic gradient information of $f_{t}$.
		\ENDFOR
		\FOR{$k=1,\dots,K$}
		\STATE $\g_{q}^{(k)}=(1-\eta_{k})\g_{q}^{(k-1)}+\eta_{k}\widetilde{\nabla}f_{t_{q}^{(k)}}(\x_{q}^{(k)})$.
		\STATE Feed back $\langle(\one-\x_{q}^{(k)})\odot \g_{q}^{(k)},\vbf_{q}^{(k)}\rangle$ as the payoff to be received by oracle $\mathcal{E}^{(k)}$.
		\ENDFOR
		\ENDFOR
	\end{algorithmic}
\end{algorithm}
\begin{theorem}[Proof in \cref{appendix:Mono}]\label{thm:2}
    Under Assumption~\ref{assumption1}-\ref{add_assumption} and $\max_{\x\in\mathcal{C}}\|\nabla f_{t}(\x)\|\le G$ for any $t\in[T]$, if we set $\eta_{k}=\frac{2}{(k+3)^{2/3}}$, when $1\le k\le\frac{K}{2}+1$, and $\eta_{k}=\frac{1.5}{(K-k+2)^{2/3}}$, when $\frac{K}{2}+2\le k\le K$, then Algorithm~\ref{alg:2} achieves:
	\begin{equation*}
	\begin{split}
		\frac{1}{e}\sum_{t=1}^{T}f_{t}(\x^{*})-\sum_{t=1}^{T}\E(f_{t}(\y_{t}))
		\le 2\mathrm{diam}(\C)(N_{1}+1)QK^{2/3}+\frac{L_{0}r^{2}(\C)}{2}Q+M_{0}\sqrt{Q}K,
	\end{split}
	\end{equation*}
	where $\x^{*}=\arg\max_{\x\in\C}\sum_{t=1}^{T}f_{t}(\x)$ and $N_{1}=\max\{5^{2/3}G^{2},8(\sigma^{2}+G^{2})+32(2G+L_{0}r(\C))^{2},4.5(\sigma^{2}+G^{2})+7(2G+L_{0}r(\C))^{2}/3\}$.
\end{theorem}
\begin{remark}
	According to Theorem~\ref{thm:2}, if we set $K=T^{3/5}$ and $Q=T^{2/5}$, the \textbf{Mono-MFW} achieves a $1/e$-regret of $O(T^{4/5})$. To the best of our knowledge, this is the first result with sublinear regret for one-shot online non-monotone DR-submodular maximization over a down-closed convex set.
\end{remark}

\subsection{Bandit Online Non-monotone Continuous DR-Submodular Maximization}\label{sec:Bandit}
\noindent In this subsection, we turn to the bandit setting for online non-monotone continuous DR-submodular maximization. To begin, we review the one-point estimator~\citep{flaxman2005online}, which is of great importance to our proposed bandit algorithm.

\subsubsection{One-point Estimator}\label{sec:one_point}
For any function $f:[0,1]^{n}\rightarrow\R_{+}$, define the $\delta$-smooth version of $f$ as $\hat{f}_{\delta}(\x)=\E_{\vbf\sim B^{n}}(f(\x+\delta\vbf))$ where $\vbf\sim B^{n}$ represents that the vector $\vbf$ is uniformly sampled from the $n$-dimensional unit ball $B^{n}$. If $\|\nabla f(\x)\|\le G$, we have $|f(\x)-\hat{f}_{\delta}(\x)|\le G\delta$. Thus, $\hat{f}_{\delta}$ can be viewed as an approximation of $f$, when $\delta$ is small. Roughly speaking, we can approximately maximize $f$ via the maximizer of $\hat{f}_{\delta}$. Note that if $f$ is continuous DR-submodular and $L_{0}$-smooth, so is $\hat{f}_{\delta}$. Moreover, according to \cite{flaxman2005online}, $\nabla\hat{f}_{\delta}(\x)=\frac{n}{\delta}\E_{\vbf\sim S^{n-1}}(f(\x+\delta\vbf)\vbf)$ where $\vbf\sim S^{n-1}$ implies that the vector $\vbf$ is uniformly sampled from the unit sphere $S^{n-1}$, which sheds light on the possibility of estimating the gradient of $\hat{f}_{\delta}(\x)$ via the function value at a random point $\x+\delta\vbf$.

However, we cannot use this estimate method directly. The point $\x+\delta\vbf$ may fall outside of the constraint set $\C$, when $\x$ is close to the boundary of $\C$. To tackle this challenge, we introduce the concept of $\delta$-interior. We say that a subset $\C^{'}$ is a $\delta$-interior of $\C$, if the ball $B(\x,\delta)$ centered at $\x$ with radius $\delta$, is included in $\C$ for any $\x\in\C^{'}$. As a result, for every point $\x\in\C^{'}$, $\x+\delta\vbf$ is included in $\C$, which enables us to use the one-point estimator. Recently, for a down-closed convex set $\C$, \citet{zhang2019online} provided a method to construct a $\delta$-interior down-closed convex set $\C^{'}$. Next, we show this outcome in Lemma~\ref{lemma:construct_set}.
\begin{lemma}[\citet{zhang2019online}]\label{lemma:construct_set}
Under Assumption~\ref{assumption1}, if there exists a positive number $r$ such that $rB^{n}_{\ge0}\subseteq\C$ where $B^{n}_{\ge0}=B^{n}\cap\R^{n}_{+}$, and $\delta<\frac{r}{\sqrt{n}+1}$, the set $\C^{'}=(1-\alpha)\C+\delta\one$ is a down-closed convex $\delta$-interior of $\C$ with $\sup_{\x\in\C,\y\in\C^{'}}\|\x-\y\|\le((\sqrt{n}+1)\frac{r(\C)}{r}+\sqrt{n})\delta$, where $\alpha=\frac{(\sqrt{n}+1)\delta}{r}$.
\end{lemma}

\subsubsection{Bandit Measured Frank-Wolfe}
\begin{algorithm}[t]
	\caption{Bandit-Measured Frank-Wolfe~(\textbf{Bandit-MFW})}
	\begin{algorithmic}[1]\label{alg:3}
		\STATE {\bf Input:} $\delta$, $r$, $\alpha=\frac{(\sqrt{n}+1)\delta}{r}$, $\delta$-interior down-closed convex set $\C^{'}=(1-\alpha)\C+\delta\one$, $K$ online linear maximization oracles on $\C^{'}$, i.e., $\mathcal{E}^{(1)},\dots,\mathcal{E}^{(K)}$, $L$, $Q=\frac{T}{L}$, $\eta_{k}$, $\g_{q}^{(0)}=\mathbf{0}$, $\x_{q}^{(0)}=\delta\one$.
		\STATE {\bf Output:} $\y_1,\dots,\y_T$.
		\FOR{$q=1,\ldots, Q$}
		\FOR{$k=1,\dots,K$}
		\STATE Receive $\vbf_{q}^{(k)}$ which is the output of oracle $\mathcal{E}^{(k)}$.
		\STATE $\tilde{\vbf}_{q}^{(k)}=(\vbf_{q}^{(k)}-\delta\one)\oslash(\one-\delta\one)$.
		\STATE $\x_{q}^{(k)}=\x_{q}^{(k-1)}+\frac{1}{K}\tilde{\vbf}_{q}^{(k)}\odot(\one-\x_{q}^{(k-1)})$.
		\ENDFOR
		\STATE Generate a random permutation $\{t_{q}^{(1)},\dots,t_{q}^{(L)}\}$ for $\{(q-1)L+1,\dots,qL\}$.
		\FOR{$t=(q-1)L+1,\dots,qL$}
		\IF{$t\in\{t_{q}^{(1)},\dots,t_{q}^{(K)}\}$}
		\STATE Play $\y_{t}=\x_{q}^{(k)}+\delta\ubf_{q}^{(k)}$ for $f_{t}$, where $\ubf_{q}^{(k)}\sim S^{n-1}$. \quad \quad \COMMENT{$\triangleright$ Exploration}
		\ENDIF
		\IF{$t\in\{(q-1)L+1,\dots,qL\}\setminus\{t_{q}^{(1)},\dots,t_{q}^{(K)}\}$}
		\STATE Play $\y_{t}=\x_{q}^{(K)}$ for $f_{t}$. \quad \quad \COMMENT{$\triangleright$ Exploitation}
		\ENDIF
		\ENDFOR
		\FOR{$k=1,\dots,K$}
		\STATE $\g_{q}^{(k)}=(1-\eta_{k})\g_{q}^{(k-1)}+\eta_{k}\frac{n}{\delta}f_{t^{(k)}_{q}}(\x_{q}^{(k)}+\delta\ubf_{q}^{(k)})\ubf_{q}^{(k)}$.
		\STATE $\tilde{\x}_{q}^{(k)}=(\x_{q}^{(k)}-\delta\one)\oslash(\mathbf{1}-\delta\one)$.
		\STATE Feed back $\langle(\one-\tilde{\x}_{q}^{(k)})\odot \g_{q}^{(k)},\vbf_{q}^{(k)}\rangle$ as the payoff to be received by oracle $\mathcal{E}^{(k)}$.
		\ENDFOR
		\ENDFOR
	\end{algorithmic}
\end{algorithm}

\noindent To design an efficient algorithm in the bandit setting, a simple idea is to replace the stochastic gradient in Algorithm~\ref{alg:2} with the one-point estimator and run it on the $\delta$-interior $\C^{'}$. However, we cannot take this simple policy directly. In Algorithm~\ref{alg:2}, for each $t$ in the $q$-th block, we play $\x_{q}^{(K)}$ for $f_{t}$, but we may require inquiring the gradient at a different point $\x_{q}^{(k)}$. Therefore, we could not construct the one-point gradient estimate at point $\x_{q}^{(k)}$ via the reward $f_{t}(\x_{q}^{(K)})$, when $k\neq K$. 

To circumvent this drawback, we take the exploration-exploitation trade-off strategy in \cite{zhang2019online}. Specifically, we divide the $T$ reward functions into $Q$ blocks of size $L$, where $T=LQ$. Then, we cut each block into two phases~(i.e., exploration and exploitation). Taking the $q$-th block as an example, in the exploration phase, we select $K$ random reward functions to play the $\x_{q}^{(k)}+\delta\ubf_{q}^{(k)}$ which provide the one-point gradient estimators. Then, in the exploitation phase, we commit to the point $\x_{q}^{(K)}$ for the remaining $(L-K)$ reward functions. Combining Algorithm~\ref{alg:2} with this strategy, we present Algorithm~\ref{alg:3}~(\textbf{Bandit-MFW}). Moreover, we make an additional assumption and provide the regret bound of Algorithm~\ref{alg:3}.

\begin{assumption}\label{assumption2}\
	\begin{enumerate}
		\item[(i)] There exists a positive number $r$ such that $rB^{n}_{\ge0}\subseteq\C$ where $B^{n}_{\ge0}=B^{n}\cup\R^{n}_{+}$.
		\item[(ii)] For each $t\in[T]$, $\sup_{x\in\C}f_{t}(\x)\le M_{1}$.
	\end{enumerate}
\end{assumption} 

\begin{theorem}[Proof in \cref{appendix:Bandit}]\label{thm:3}
	Under Assumption~\ref{assumption1}, \ref{assumption2}, and $\max_{\x\in\mathcal{C}}\|\nabla f_{t}(\x)\|\le G$ for any $t\in[T]$, if we set $\eta_{k}=\frac{2}{(k+3)^{2/3}}$ for $k\in[K]$, then Algorithm~\ref{alg:3} achieves:
	\begin{equation*}
	\begin{split}
		\ &\frac{1}{e}\sum_{t=1}^{T}f_{t}(\x^{*})-\sum_{t=1}^{T}\E(f_{t}(\y_{t}))
		\le
		C_{1}\frac{LQ}{K}+M_{0}L\sqrt{Q} \\ 
		\ &+\frac{C_{2}LQ}{2\delta K^{1/3}} 
		+\frac{C_{3}\delta LQ}{2K^{1/3}}+2M_{1}KQ+C_{4}T\delta ,
	\end{split}	
	\end{equation*}
	where $\x^{*}=\arg\max_{\x\in\C}\sum_{t=1}^{T}f_{t}(\x)$, $C_{1}=\frac{L_{0}r^{2}(\C)}{2}$, $C_{2}=(8n^{2}M_{1}^{2}+1)\mathrm{diam}(\C)$, $C_{3}=\max\{3^{2/3}G^{2},8G^{2}+3(4.5L_{0}r(\C)+3G)^{2}/2\}\mathrm{diam}(\C)$ and $C_{4}=((\sqrt{n}+1)\frac{r(\C)}{r}+\sqrt{n}+2)G$.
\end{theorem}

\begin{remark}
	According to Theorem~\ref{thm:3}, if we set $L=T^{7/9}$, $Q=T^{2/9}$, $K=T^{2/3}$, and $\delta=\frac{r}{(\sqrt{n}+2)T^{1/9}}$, \textbf{Bandit-MFW} achieves a $1/e$-regret of $O(T^{8/9})$ . As far as we know, this is the first sublinear-regret online algorithm for continuous non-monotone DR-submodular maximization with bandit feedback.
\end{remark}

\section{Empirical Evaluation}\label{sec:Experiment}
In this section, we compare the performance of the following algorithms with the help of CVX optimization tool~\citep{grant2014cvx}:

\noindent\textbf{Meta-Measured Frank-Wolfe~($\beta$-Meta)}: In Algorithm~\ref{alg:1}, we set $K=T^{\beta}$ and $\eta_{k}=\frac{2}{(k+3)^{2/3}}$ for any $k\in[K]$. In the experiments, we consider $\beta=\frac{3}{4}$ or $\beta=\frac{3}{2}$.\\
\noindent\textbf{Mono-Measured Frank-Wolfe~(Mono)}: In Algorithm~\ref{alg:2}, we set $K=T^{3/5}$ and $Q=T^{2/5}$. Simultaneously, $\eta_{k}=\frac{2}{(k+3)^{2/3}}$ for any $1\le k\le\frac{K}{2}+1$ and $\eta_{k}=\frac{1.5}{(K-k+2)^{2/3}}$ for any $\frac{K}{2}+2\le k\le K$.\\
\noindent\textbf{Bandit-Measured Frank-Wolfe~(Bandit)}: In Algorithm~\ref{alg:3}, we set $L=T^{7/9}$, $Q=T^{2/9}$, $K=T^{2/3}$, $\delta=\frac{r}{(\sqrt{n}+2)T^{1/9}}$ as well as $\eta_{k}=\frac{2}{(k+3)^{2/3}}$ for any $k\in[K]$.\\
\noindent\textbf{Online algorithm for down-closed convex sets~(ODC)}: We consider Algorithm 2 in \citep{thang2021online} where  $L=T^{3/4}$ and $\rho_{l}=\frac{2}{(l+3)^{2/3}}$ for all $1\le l\le L$.

\subsection{Non-Convex/Non-Concave Quadratic Programming}
We consider the quadratic objective $f(\x) = \frac{1}{2}\x^{T}\Hbf\x + \mathbf{h}^{T}\x+c$ and constraints $\mathcal{C}=\{\x\in \R^{n}_{+} | \mathbf{A}\x\le\mathbf{b}, \mathbf{0}\le\x\le\ubf, \mathbf{A}\in \R^{m\times n}_{+}, \mathbf{b}\in\R_{+}^{m}\}$. Following \cite{bian2017continuous,bian2017guaranteed,chen2018online}, we choose the matrix $\boldsymbol{H}\in \mathbb{R}^{n\times n}$ to be a randomly generated symmetric matrix with entries $H_{ij}$ uniformly distributed in $[-10,0]$, and the matrix $\boldsymbol{A}$ to be a random matrix with entries uniformly distributed in $[0,1]$. It can be verified that $f$ is a continuous DR-submodular function and $P$ is down-closed. We set $\mathbf{b}=\ubf=\one$. Meanwhile, we set $\mathbf{h}=-0.1*\Hbf^{T}\ubf$, which ensures the non-monotone property. To make $f$ non-negative, we choose $c=-0.5*\sum_{i,j}H_{ij}$. We consider the  Gaussian noise for gradient, i.e., $(\widetilde{\nabla}f_{t}(\boldsymbol{x}))_{i}=(\nabla f_{t}(\boldsymbol{x}))_{i}+\delta\mathcal{N}(0,1)$ for any $i\in[n]$ and $\x\in[0,1]^{n}$, where we set $\delta=0.1$ in the experiments.

In our simulations, we first generate $T=200$ reward functions $f_{1},\dots,f_{T}$ with associated matrices $\Hbf_{1},\dots,\Hbf_{T}$. Next, we run the well-studied offline algorithms \citep{bian2017continuous,mitra2021submodular+} to produce an effective solution $\x^{*}_{t}$ that is a $(1/e)$-approximation to the optimum of the objective $\sum_{m=1}^{t}f_{m}$ for each $t\in[T]$. Then, under different $n$ and $m$, we present the trend of the ratio between regret and horizon, namely, $(\sum_{m=1}^{t}f_{m}(\x^{*}_{t})-\sum_{m=1}^{t}f_{m}(\y_{m}))/t$ in Figure~\ref{fig1:(a)}-\ref{fig1:(c)}. 
Simultaneously, we report the $200$-round running time in Table~\ref{tab:2}.

As shown in Figure~\ref{fig1}, our proposed Meta-MFW with $\beta=3/2$ and $3/4$~(i.e., $3/2$-Meta and $3/4$-Meta) achieve lower regret in contrast with ODC~\citep{thang2021online}. 
Interestingly, the regret curves of both $3/2$-Meta and $3/4$-Meta are nearly the same in Figure~\ref{fig1}. When the iteration index increases, Mono~(Algorithm~\ref{alg:2}) also outperforms ODC in all three settings.
Moreover, according to Table~\ref{tab:2}, 3/4-Meta and Mono effectively save running time compared with ODC. For example, when $n=50,m=50$, we spend $16.61$ and $0.31$ seconds in running 3/4-Meta and Mono, respectively, while the ODC takes $45.08$ seconds. 
It is worth mentioning that although the bandit algorithm~(Algorithm~\ref{alg:3}) with only one-point reward information exhibits the lowest convergence rate among all algorithms, it has the least running time as demonstrated in Table~\ref{tab:2}.

\begin{figure*}[t]
\subfigure[$n=25,m=15$\label{fig1:(a)}]{\includegraphics[width=0.32\linewidth]{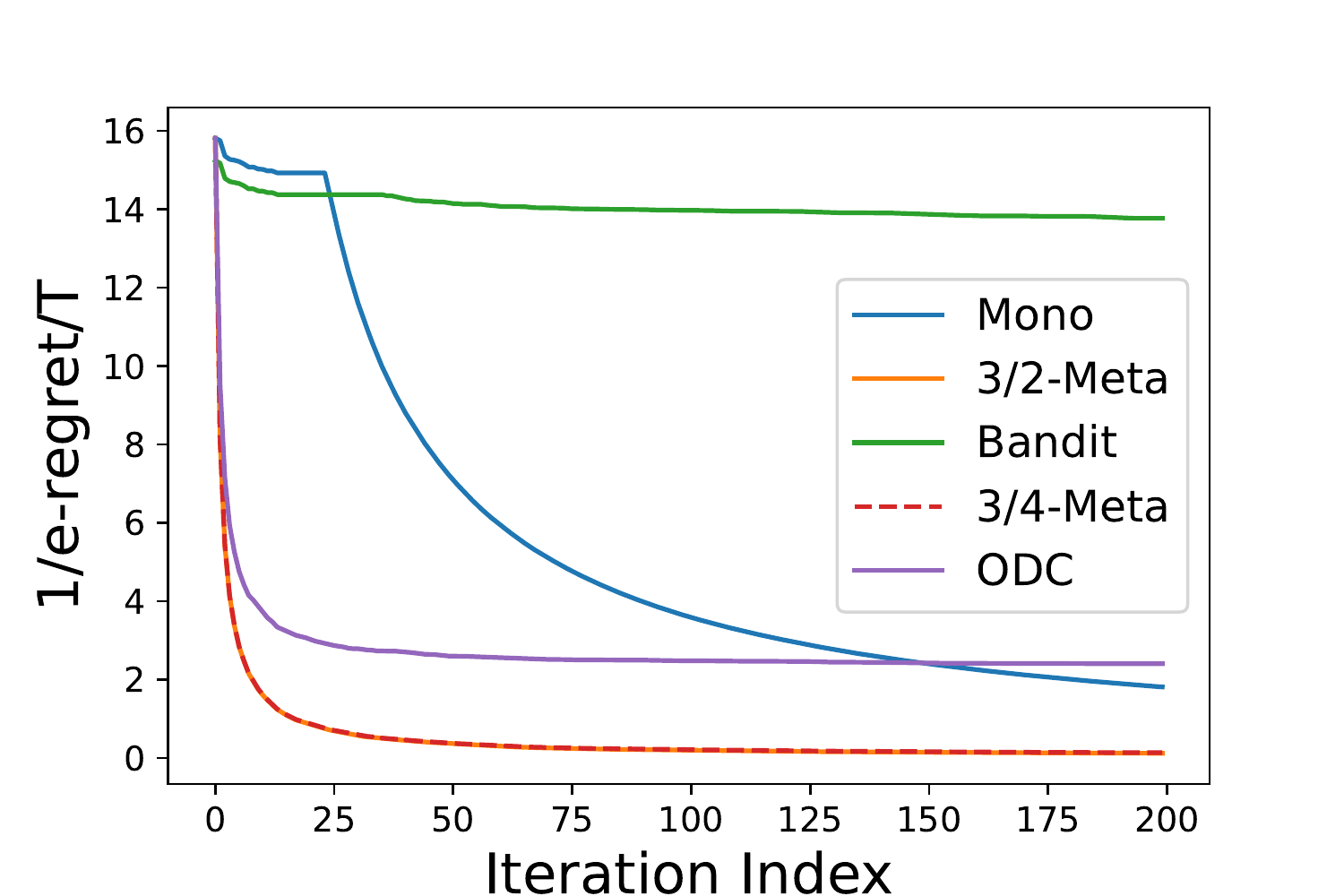}}
\subfigure[$n=40,m=20$\label{fig1:(b)}]{\includegraphics[width=0.32\linewidth]{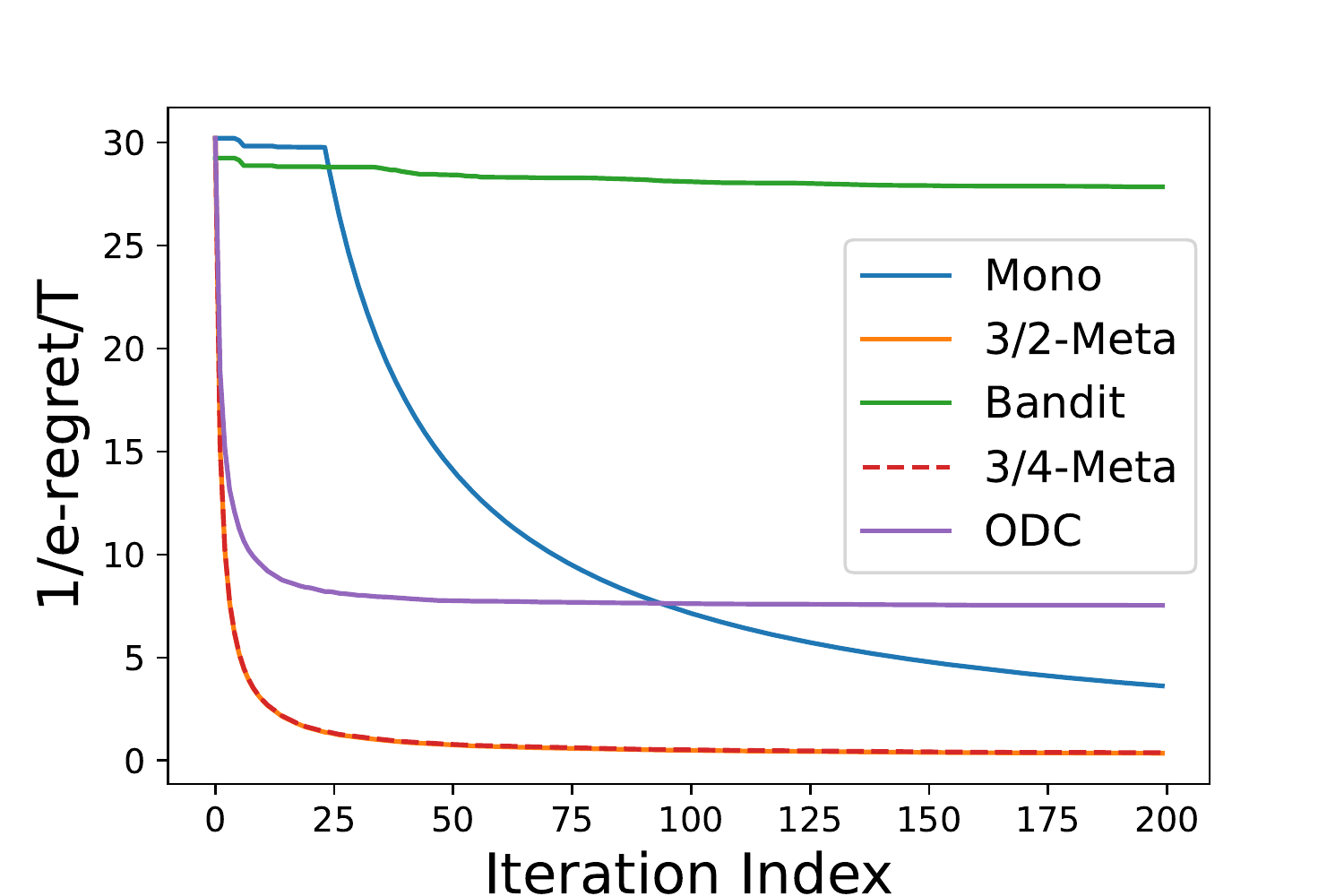}}
\subfigure[$n=50,m=50$\label{fig1:(c)}]{\includegraphics[width=0.32\linewidth]{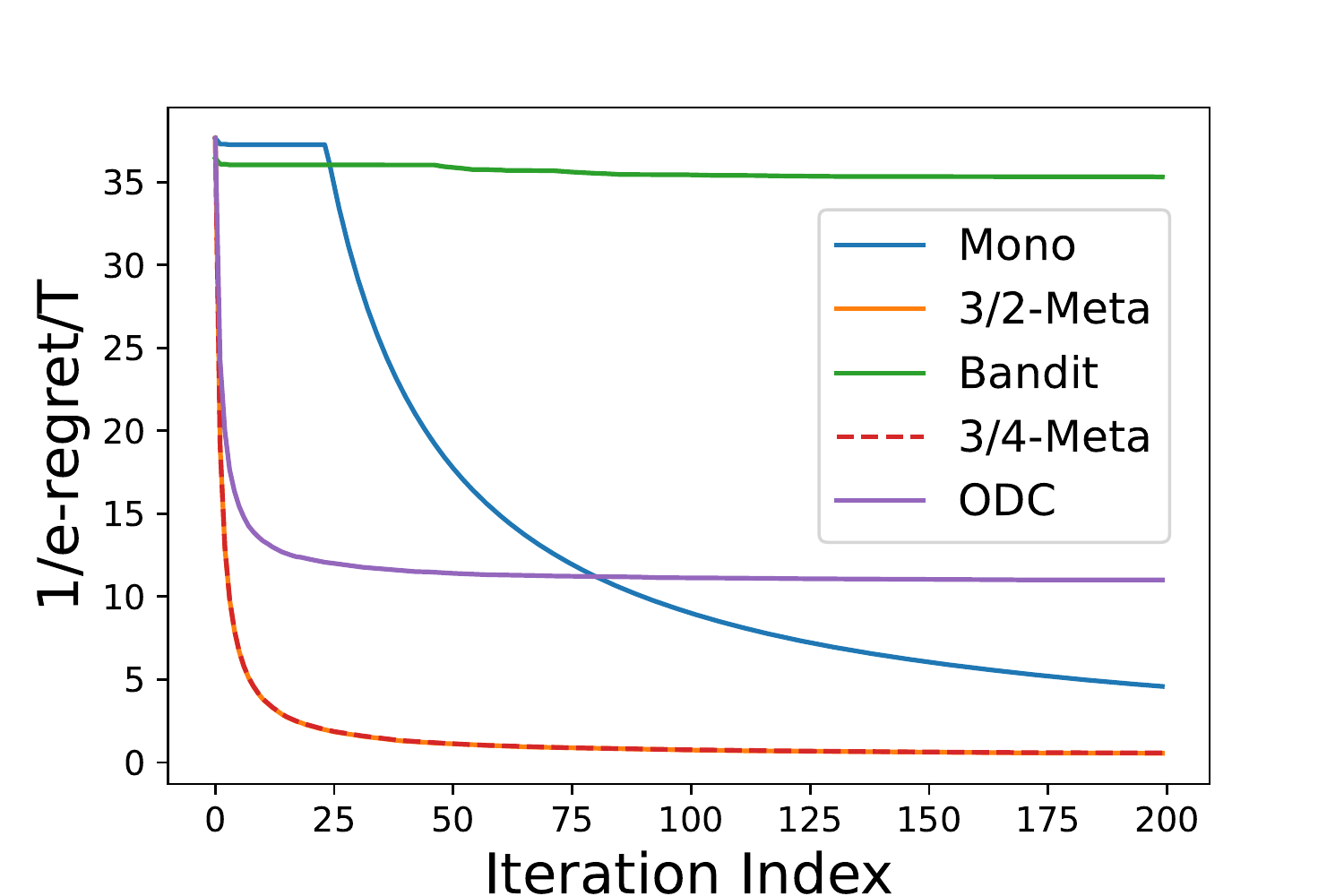}}
\caption{We test the performance of the \textbf{3/2-Meta}, \textbf{3/4-Meta}, \textbf{Mono}, \textbf{Bandit}, and \textbf{ODC} in the simluated continuous DR-submodular quadratic programming under different dimension $n$ and number of linear constraints $m$.}
\label{fig1}
\end{figure*}

\begin{table}[t]
	\centering
	\caption{Running time~(in seconds)}
    \begin{tabular}{c|c|c|c}
    \toprule[1.5pt]
    \diagbox{Method }{$(n,m)$}& $(25,15)$ & $(40,20)$ & $(50,50)$\\
    \hline 
    ODC & $14.14$ & $26.38$ & $45.08$\\
    \hline
    \hline
    3/2-Meta & $471.99$ & $609.50$ & $895.97$\\
    \hline 
    3/4-Meta & $8.75$ & $11.54$ & $16.61$\\

    \hline 
    Mono & $0.16$ & $0.21$ & $0.31$\\
    \hline
    Bandit & $0.11$ & $0.14$ & $0.21$\\
    \midrule[1.5pt]
    \end{tabular}
	\label{tab:2}
\end{table}

\subsection{Revenue Maximization}

\begin{figure*}[t]
\subfigure[CA-HepPH\label{fig2:(a)}]{\includegraphics[width=0.32\linewidth]{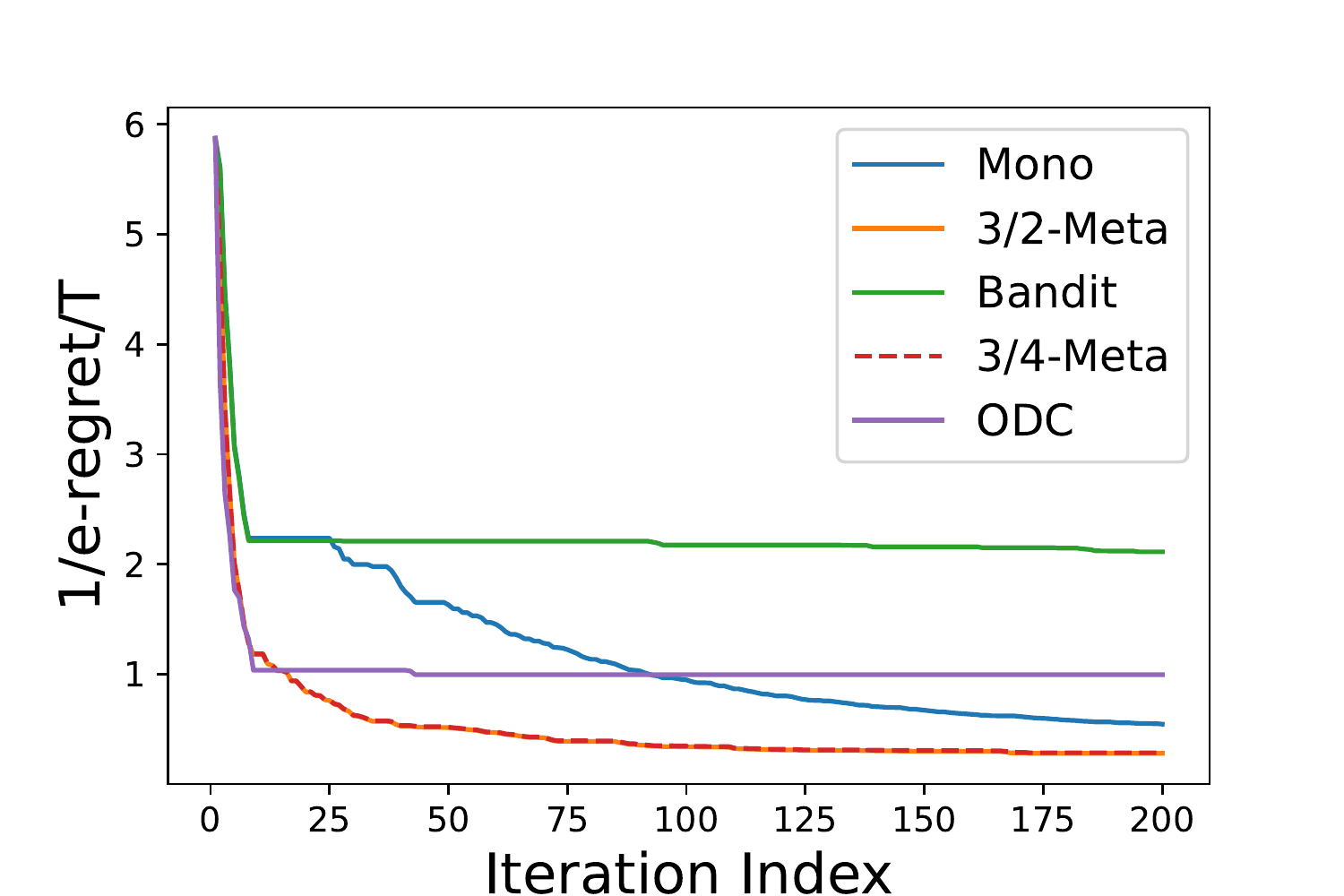}}
\subfigure[CA-GrQc\label{fig2:(b)}]{\includegraphics[width=0.32\linewidth]{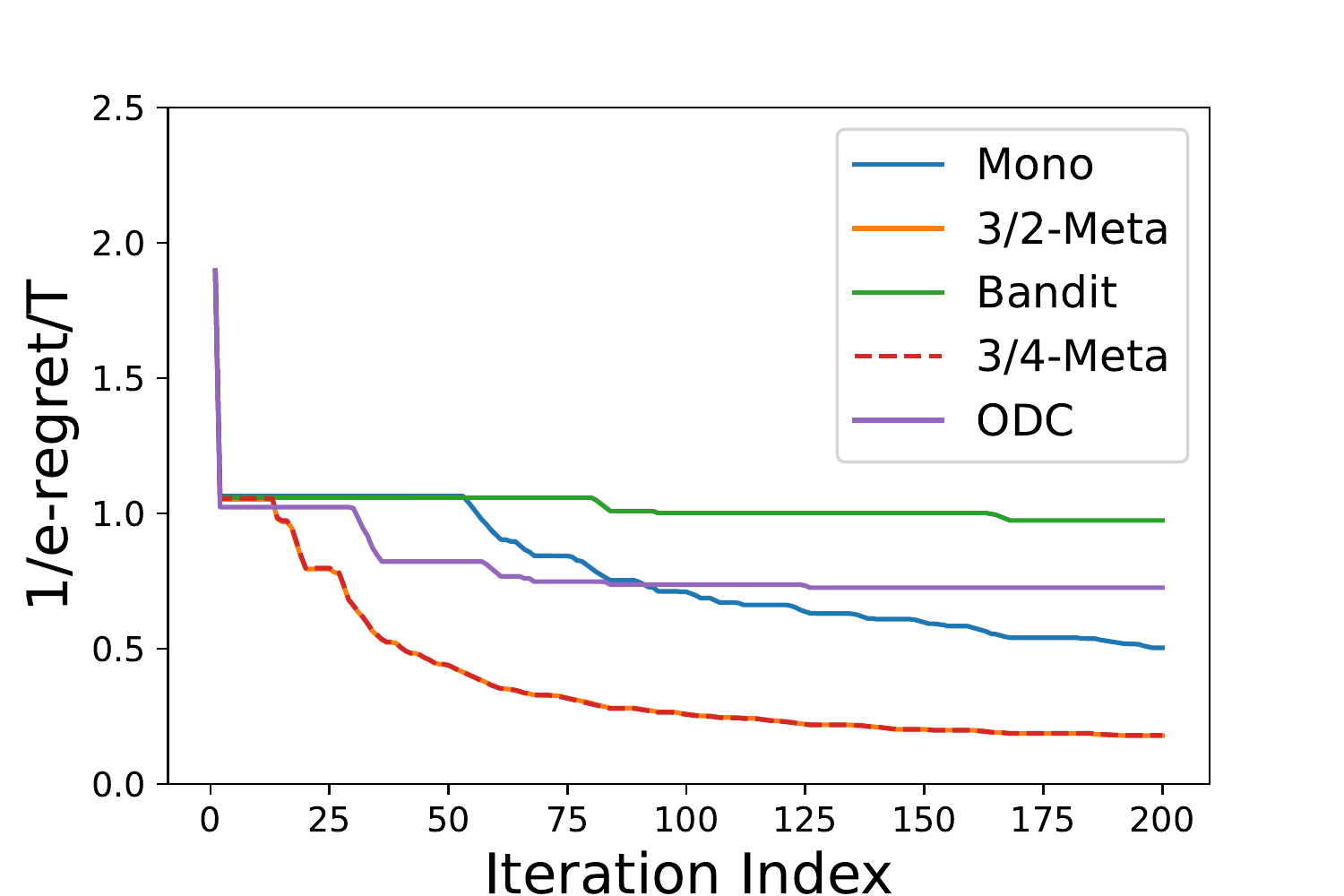}}
\subfigure[CA-HeTPH\label{fig2:(c)}]{\includegraphics[width=0.32\linewidth]{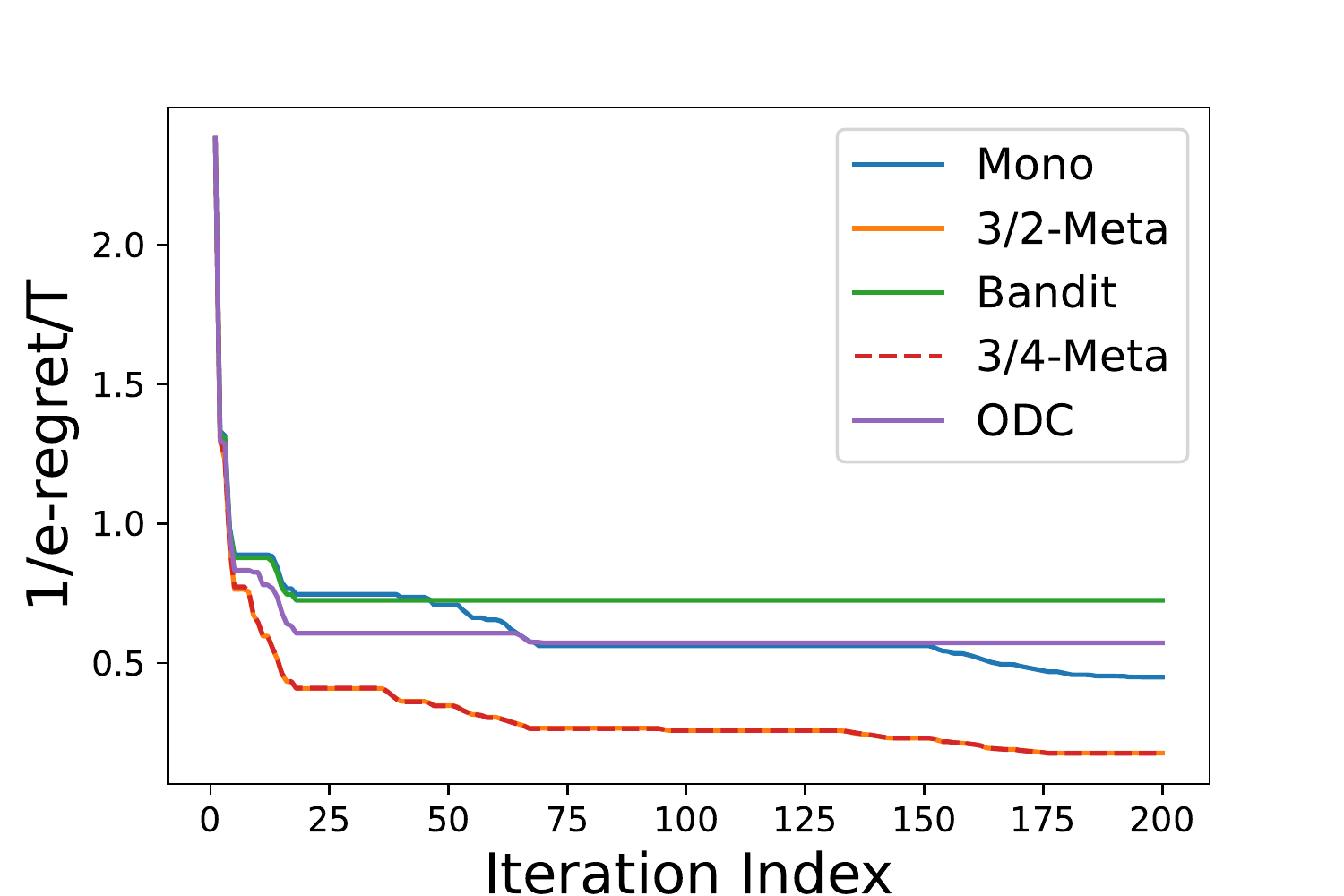}}
\caption{We test the performance of the \textbf{3/2-Meta}, \textbf{3/4-Meta}, \textbf{Mono}, \textbf{Bandit}, and \textbf{ODC} in revenue maximization on social network \textbf{CA-HepPH}, \textbf{CA-GrQc} and \textbf{CA-HepTH}.}
\label{fig2}
\end{figure*}

\begin{table}[t]
	\centering
	\caption{Running time~(in seconds)}
    \begin{tabular}{c|c|c|c}
    \toprule[1.5pt]
    Method & CA-HepPH & CA-GrQc & CA-HepTH\\
    \hline 
    ODC & $51.70$ & $94.21$&$161.27$\\
    \hline
    \hline
    3/2-Meta & $1302.50$ & $1818.33$& $3156.41$\\
    \hline 
    3/4-Meta & $24.97$ & $35.26$&$59.56$\\
    \hline 
    Mono & $0.52$ & $0.64$&$1.08$\\
    \hline
    Bandit & $0.24$ & $0.33$&$0.68$ \\
    \midrule[1.5pt]
    \end{tabular}
	\label{tab:3}
\end{table}

In this application, we consider revenue maximization on an undirected social network $G=(V,W)$ where $V$ is the set of nodes, and $w_{ij}\in W$ represents the weight of the edge between node $i$ and node $j$. If we invest $x$ proportion of the budget $B$ on a user~(node) $i\in V$, the user becomes an advocate of some product with probability $1-(1-p)^{xB}$, where $p\in(0,1)$ is a parameter. Intuitively, for investing a unit cost to user~(node) $i$, we have an extra chance that the user $i$ becomes an advocate with probability $p$. Let $S\subseteq V$ be a random set of users who advocate the product. Following~\citet{thang2021online}, the revenue with respect to $S$ is defined as $\sum_{i\in S}\sum_{j\in V\setminus S}w_{ij}$. Let $f: [0,1]^{|V|}\rightarrow\R_{+}$ be the expected revenue obtained in this model, that is $f(\x)=\sum_{i}\sum_{j\neq i}w_{ij}(1-(1-p)^{(\x)_{i}B})(1-p)^{(\x)_{j}B}$. It has been shown that $f$ is a non-monotone continuous DR-submodular function~\citep{soma2017non}.

In our experiments, we first sample three subgraphs from social networks~\citep{leskovec2007graph} to simulate the online revenue maximization, i.e., a part of Arxiv Hep-Ph (High Energy Physics - Phenomenology) collaboration network (316 edges and 56 vertices), a part of Arxiv Gr-Qc (General Relativity and Quantum Cosmology) collaboration network (316 edges and 81 vertices) as well as a part of Arxiv Hep-Th (High Energy Physics - Theory) collaboration network (658 edges and 106 vertices). At each round $t\in[T]$, we randomly select 20 vertices $V_{t}\subseteq V$ and construct $W_{t}$ with edge-weight $w^{t}_{ij}=100$ if $i,j\in V_{t}$ and edge $(i,j)$ exists in the network. If $i$ or $j$ is not in $V_{t}$, $w^{t}_{ij}=0$. As a result, the reward function $f_{t}(\x)=\sum_{i}\sum_{j\neq i}w^{t}_{ij}(1-(1-p)^{(\x)_{i}B})(1-p)^{(\x)_{j}B}$.
We also impose a constraint as $\mathcal{C}=\{\x\in \R^{n}_{+} | \mathbf{A}\x\le\mathbf{b}, \sum_{i}(\x)_{i}\le 1,  \mathbf{0}\le\x\le\one\}$ where $\boldsymbol{A}$ is a random matrix with entries uniformly distributed in $[0,1]$. We set $p=0.002$, $m=25$ as well as $B=5$. Similarly, we consider the  Gaussian noise for gradient with $\delta=0.1$. Then, we report the trend of the ratio between regret and time horizon in Figure~\ref{fig2:(a)}-\ref{fig2:(c)} and running time in Table~\ref{tab:3}. 

As shown in Figure~\ref{fig2}, our proposed \textbf{Meta-MFW} with $\beta=3/2$ and $3/4$~(i.e., $3/2$-Meta and $3/4$-Meta) have nearly the same curves and outperform the ODC~\citep{thang2021online}. 
Similarly, compared to the \textbf{ODC}, the \textbf{Mono-MFW}~(Algorithm~\ref{alg:2}) achieves lower regret in all three real-world social networks, when $T$ is large. 
Moreover, according to Table~\ref{tab:3}, our proposed 3/4-Meta and Mono take less running time than the \textbf{ODC} algorithm. Note that the bandit algorithm~(Algorithm~\ref{alg:3}) exhibits the lowest convergence rate among all algorithms with the fastest running time. 

\section{Conclusion}
In this paper, we design three online no-regret algorithms for non-monotone continuous DR-submodular maximization over a down-closed convex set. The first one, \textbf{Meta-MFW}, attains a $1/e$-regret bound of $O(\sqrt{T})$ while requiring inquiring the $T^{3/2}$ amounts of gradient evaluations for each reward function. The second one, \textbf{Mono-MFW}, reduces the number of per-function gradient evaluations from $T^{3/2}$ to 1, and achieves a $1/e$-regret bound of $O(T^{4/5})$. Finally, we present the \textbf{Bandit-MFW} algorithm, which is the first bandit algorithm for online continuous non-monotone DR-submodular maximization over a down-closed convex set and achieves a $1/e$-regret bound of $O(T^{8/9})$. Numerical experiments demonstrate the superior performance of our algorithms.

\bibliography{references}
\appendix
\section{Varience Reduction Techniques}\label{appendix:VR}
Our algorithms rely on the well-studied variance reduction techniques in \citep{chen2018projection,zhang2019online,mokhtari2020stochastic}. Next, we demonstrate some results about variance reduction in the following lemmas.
\begin{lemma}[\citet{chen2018projection,mokhtari2020stochastic}]\label{lemma:vr1}
Let $\{\a_{t}\}_{t=0}^{K}$ be a sequence of points in $\R^{n}$ such that $\|\a_{t}-\a_{t-1}\|\le\frac{G}{t+s}$ for all $1\le t\le K$ with fixed constant $G\ge0$ and $s\ge 3$. Let $\{\widetilde{\a}_{t}\}_{t=0}^{K}$ be a sequence of random variables such that $\E(\widetilde{\a}_{t}|\F_{t-1})=\a_{t}$ and $\E(\|\widetilde{\a}_{t}-\a_{t}\|^{2}|\F_{t-1})\le\sigma^{2}$ for every $t\ge0$, where $\F_{t-1}$ is the $\sigma$-field generated by $\{\widetilde{\a}_{k}\}_{k=0}^{t-1}$ and $\F_{0}=\emptyset$. Let $\{\dbf_{t}\}_{t=0}^{K}$ be a sequence of random variables where $\dbf_{0}$ is fixed and subsequent $\dbf_{t}$ are obtained by $\dbf_{t}=(1-\eta_{t})\dbf_{t-1}+\eta_{t}\widetilde{\a}_{t}$. If we set $\eta_{t}=\frac{2}{(t+s)^{2/3}}$, we have
\begin{equation}
    \E(\|\dbf_{t}-\a_{t}\|^{2})\le\frac{N}{(t+s+1)^{2/3}},
\end{equation}
where $N=\max\{\|\a_{0}-\dbf_{0}\|^{2}(s+1)^{2/3},4\sigma^{2}+3G^{2}/2\}$.
\end{lemma}

\begin{lemma}[\citet{zhang2019online}]\label{lemma:vr2}
Let $\{a_{t}\}_{t=0}^{K}$ be a sequence of points in $\R^{n}$ such that $\|\a_{t}-\a_{t-1}\|\le\frac{G}{K+2-t}$ for all $1\le t\le K$ with fixed constant $G\ge0$. Let $\{\widetilde{\a}_{t}\}_{t=0}^{K}$ be a sequence of random variables such that $\E(\widetilde{\a}_{t}|\F_{t-1})=\a_{t}$ and $\E(\|\widetilde{\a}_{t}-\a_{t}\|^{2}|\F_{t-1})\le\sigma^{2}$ for every $t\ge0$, where $\F_{t-1}$ is the $\sigma$-field generated by $\{\widetilde{\a}_{k}\}_{k=0}^{t-1}$ and $\F_{0}=\emptyset$. Let $\{\dbf_{t}\}_{t=0}^{K}$ be a sequence of random variables where $\dbf_{0}$ is fixed and subsequent $\dbf_{t}$ are obtained by $\dbf_{t}=(1-\eta_{t})\dbf_{t-1}+\eta_{t}\widetilde{\a}_{t}$. If we set $\eta_{t}=\frac{2}{(t+3)^{2/3}}$, when $1\le t\le\frac{K}{2}+1$, and when $\frac{K}{2}+2\le t\le K$, $\eta_{t}=\frac{1.5}{(K-t+2)^{2/3}}$, we have
\begin{equation}
    \E(\|\dbf_{t}-\a_{t}\|^{2})\le\left\{\begin{aligned}
       &\frac{N}{(t+4)^{2/3}}& 1\le t\le\frac{K}{2}+1\\
       &\frac{N}{(K-t+1)^{2/3}}& \frac{K}{2}+2\le t\le K
    \end{aligned}\right.
\end{equation} 
where $N=\max\{5^{2/3}\|\a_{0}-\dbf_{0}\|^{2},4\sigma^{2}+32G^{2},2.25\sigma^{2}+7G^{2}/3\}$.
\end{lemma}

\section{Proofs in \cref{sec:Meta}}\label{appendix:Meta}
\subsection{The Properties of New Update Rule}
In our Algorithm~\ref{alg:1}, we take a novel update rule (Equation~(\ref{equ:update})). Before going into the detail, we first demonstrate some important properties of this new update rule. Next, we use new symbols to retell this update rule: Given a series of update directions $\dbf_{k}\in\C,\forall k\in[K]$ and initial point $\y_{0}=\mathbf{0}$, we consider the following update rule, i.e., 
\begin{equation}\label{equ:update1}
   \y_{k}=\y_{k-1}+\frac{1}{K}\dbf_{k}\odot(\one-\y_{k-1}).
\end{equation}
A prompt benefit of this rule is shown in the following lemma.
\begin{lemma}
When $\C\subseteq[0,1]^{n}$ is down-closed convex set and $\mathbf{0}\in\C$, then $\y_{k}\in\C$ for any $k\in[K]$.
\end{lemma}
\begin{proof}
First, we prove that $\y_{k}\le\mathbf{1}$ for any $k\le K$. By induction, we know $\y_{0}=\mathbf{0}$. If we assume $\y_{k-1}\le\mathbf{1}$, then
\begin{equation*}
    \begin{aligned}
     \y_{k}&=\y_{k-1}+\frac{1}{K}\dbf_{k}\odot(\one-\y_{k-1})\\
     =&\frac{1}{K}\dbf_{k}+\y_{k-1}\odot(\one-\frac{1}{K}\dbf_{k})\\
     \le&\frac{1}{K}\dbf_{k}+\one-\frac{1}{K}\dbf_{k}\\
     =&\one.
    \end{aligned}
\end{equation*} 
As a result, $\y_{k}\le\mathbf{1}$ for any $k\le K$. Next, we verify that $\y_{k}\in\mathcal{C}$. According to \cref{equ:update1}, we could conclude that $\y_{K}=\frac{1}{K}\sum_{k=1}^{K}\dbf_{k}\odot(\one-\y_{k-1})$. Due to convexity and each $\dbf_{k}\in\mathcal{C}$, we know $\frac{1}{K}\sum_{k=1}^{K}\dbf_{k}\in\C$. Also, we know that $\mathbf{0}\le\y_{1}\le\y_{2}\le\dots\y_{K}\le\frac{1}{K}\sum_{k=1}^{K}\dbf_{k}$ ($\y_{k}\le\mathbf{1}$) so that $\y_{k}\in\C$ for any $k\in[K]$(the down-closed property).
\end{proof}
Moreover, we could derive a upper bound about every element of $\y_{k}$, i.e., 
\begin{lemma}\label{lemma:3}
For $i\in[n]$ and $k\in[K]$, we have $(\y_{k})_{i}\le1-(1-\frac{1}{K})^{k}$.
\end{lemma}
\begin{proof}
From \cref{equ:update1}, we have
\begin{equation}
    \begin{aligned}
 (\y_{k})_{i}&=(\y_{k-1})_{i}+\frac{1}{K}(\dbf_{k}\odot(\mathbf{1}-\y_{k-1}))_{i}\\
 &=(\y_{k-1})_{i}+\frac{1}{K}(\dbf_{k})_{i}*(1-(\y_{k-1})_{i})\\
 &\le(\y_{k-1})_{i}+\frac{1}{K}(1-(\y_{k-1})_{i})\\
 &=(1-\frac{1}{K})(\y_{k-1})_{i}+\frac{1}{K},
    \end{aligned}
\end{equation} where the inequality follows from $(\dbf_{k})_{i}\le1$ and $(\y_{k-1})_{i}\le 1$ .

First, we have $(\y_{0})_{i}=0\le0$. If $(\y_{k})_{i}\le1-(1-\frac{1}{K})^{k}$, we have
\begin{equation}
    \begin{aligned}
 (\y_{k})_{i}&\le(1-\frac{1}{K})(\y_{k-1})_{i}+\frac{1}{K}\\
 &\le(1-\frac{1}{K})(1-(1-\frac{1}{K})^{k})+\frac{1}{K}\\
 &=1-(1-\frac{1}{K})^{k+1}.
    \end{aligned}
\end{equation}
Therefore, we have $(\y_{k})_{i}\le1-(1-\frac{1}{K})^{k}$ by induction.
\end{proof}Next, for any continuous DR-submodular function $f:[0,1]^{n}\rightarrow\R_{+}$, we show the relationship between $f(\z)$ and $f(\x)$ when the vector $\z$ take a similar form of the update rule (Equation~\eqref{equ:update1}), namely, $\z=\y+(\one-\y)\odot\x$ where $\x,\y\in[0,1]^{n}$. Noticeably, $\z\ge\x$.
\begin{lemma}\label{lemma:4}
For any continuous DR-submodular function $f:[0,1]^{n}\rightarrow\R_{+}$, when $\z=\y+(\one-\y)\odot\x$ where $\x,\y\in[0,1]^{n}$, we have 
\begin{equation*}
     f(\z)\ge(1-\|\y\|_{\infty})f(\x).
\end{equation*}
\end{lemma}
\begin{proof}
First, we set $g(z)=f(\x+z(\one-\x)\odot\y)$. Moreover, we know $\x+\frac{1}{\|\y\|_{\infty}}(\one-\x)\odot\y\in[0,1]^{n}$. According to \citep{bian2020continuous,thang2021online}, we know continuous DR-submodular function $f$ is concave along the any positive direction. Therefore, $g$ is a concave function in the interval $[0,\frac{1}{\|\y\|_{\infty}}]$.
 As a result, we have
\begin{equation}
    \begin{aligned}
     f(\y+(\one-\y)\odot\x)&=f(\x+(\one-\x)\odot\y)\\
     &=g(1)\\
     &=g(\|\y\|_{\infty}*\frac{1}{\|\y\|_{\infty}}+(1-\|\y\|_{\infty})*0)\\
     &\ge(1-\|\y\|_{\infty})g(0)+\|\y\|_{\infty}g(\frac{1}{\|\y\|_{\infty}})\\
     &\ge(1-\|\y\|_{\infty})g(0)\\
     &=(1-\|\y\|_{\infty})f(\x),
    \end{aligned}
\end{equation}where the first inequality comes from the concave property of $g$; the second from $g(\frac{1}{\|\y\|_{\infty}})\ge0$.
\end{proof}
Thus, according to Equation~\eqref{equ:update1} and \cref{lemma:3}-\ref{lemma:4}, we have $f(\y_{k}+(\one-\y_{k})\odot\y^{*})\ge(1-\|\y_{k}\|_{\infty})f(\y^{*})\ge(1-\frac{1}{K})^{k}f(\y^{*})$ where $\y^{*}=\arg\max_{\y\in\C}f(\y)$, which sheds light on the possibility to derive a constant-factor approximation for maximizing a non-monotone DR-submodular function for our proposed algorithms.

\subsection{Proof of Theorem~\ref{thm:1}}
First, we present a frequently used lemma.
\begin{lemma}\label{lemma:use}
For any continuous DR-submodular function $f:[0,1]^{n}\rightarrow\R_{+}$ with smoothness parameter $L_{0}$, if $\x_{k}=\x_{k-1}+\frac{1}{K}\vbf_{k}\odot(\one-\x_{k-1})$ for any $0\le k\le K$, then we have, for  $\forall\dbf\in\R^{n}$ and $\forall\y\in[0,1]^{n}$,
\begin{equation}
    \begin{aligned}
     f(\x_{k})\ge & (1-\frac{1}{K})f(\x_{k-1})+\frac{1}{K}f(\x_{k-1}+(\one-\x_{k-1})\odot\y)+\frac{1}{K}\langle (\one-\x_{k-1})\odot \dbf,\vbf_{k}-\y\rangle\\
     &+\frac{1}{K}\langle(\vbf_{k}-\y)\odot(\one-\x_{k-1}),\nabla f(\x_{k-1})-\dbf\rangle-\frac{L_{0}}{2}\|\x_{k}-\x_{k-1}\|^{2}.
    \end{aligned}
\end{equation}
\end{lemma}
\begin{proof}
According to the $L_{0}$-smooth condition, we have
\begin{equation}
\begin{aligned}
f(\x_{k})-f(\x_{k-1})&\ge\langle \x_{k}-\x_{k-1},\nabla f(\x_{k-1})\rangle-\frac{L_{0}}{2}\|\x_{k}-\x_{k-1}\|^{2}\\
 &=\frac{1}{K}\langle \vbf_{k}\odot(\one-\x_{k-1}),\nabla f(\x_{k-1})\rangle-\frac{L_{0}}{2}\|\x_{k}-\x_{k-1}\|^{2}.
\end{aligned}
\end{equation}
Then,
\begin{equation}
    \begin{aligned}
    \langle &\vbf_{k}\odot(\one-\x_{k-1}),\nabla f(\x_{k-1})\rangle\\
    = & \langle \vbf_{k}\odot(\one-\x_{k-1}),\dbf\rangle+\langle\vbf_{k}\odot(\one-\x_{k-1}),\nabla f(\x_{k-1})-\dbf\rangle\\
    = & \langle(\one-\x_{k-1})\odot \dbf,\vbf_{k}\rangle+\langle \vbf_{k}\odot(\one-\x_{k-1}),\nabla f(\x_{k-1})-\dbf\rangle\\
    = & \langle (\one-\x_{k-1})\odot \dbf,\y\rangle+\langle (\one-\x_{k-1})\odot \dbf,\vbf_{k}-\y\rangle+\langle \vbf_{k}\odot(\one-\x_{k-1}),\nabla f(\x_{k-1})-\dbf\rangle\\
    = & \langle (\one-\x_{k-1})\odot\nabla f(\x_{k-1}),\y\rangle+\langle (\one-\x_{k-1})\odot \dbf,\vbf_{k}-\y\rangle+\langle(\vbf_{k}-\y)\odot(\one-\x_{k-1}),\nabla f(\x_{k-1})-\dbf\rangle\\
    = & \langle\nabla f(\x_{k-1}), (\one-\x_{k-1})\odot \y\rangle+\langle (\one-\x_{k-1})\odot \dbf,\vbf_{k}-\y\rangle+\langle(\vbf_{k}-\y)\odot(\one-\x_{k-1}),\nabla f(\x_{k-1})-\dbf\rangle.
    \end{aligned}
\end{equation}
 For DR-submodular function $f$, we also have
 \begin{equation}
     \langle\nabla f(\x_{k-1}), (\one-\x_{k-1})\odot\y\rangle\ge f(\x_{k-1}+(\one-\x_{k-1})\odot\y)-f(\x_{k-1}),
 \end{equation} because $f$ is concave along the direction $(\one-\x_{k-1})\odot\y$.
Finally, we have
 \begin{equation}
    \begin{aligned}
     f(\x_{k})\ge & (1-\frac{1}{K})f(\x_{k-1})+\frac{1}{K}f(\x_{k-1}+(\one-
     x_{k-1})\odot\y)+\frac{1}{K}\langle (\one-\x_{k-1})\odot \dbf,\vbf_{k}-\y\rangle\\
     &+\frac{1}{K}\langle(\vbf_{k}-\y)\odot(\one-\x_{k-1}),\nabla f(\x_{k-1})-\dbf\rangle-\frac{L_{0}}{2}\|\x_{k}-\x_{k-1}\|^{2}.
    \end{aligned}
\end{equation}
\end{proof}
Then, we show how $\g_{t}^{(k)}$~(See Line 10 in \cref{alg:1}) approximates the gradient $\nabla f_{t}(\x_{t}^{(k)})$.
\begin{lemma} 
Under Assumption~\ref{assumption1}, if we set $\eta_{k}=\frac{2}{(k+3)^{2/3}}$ for any $k\in[K]$, then we have, for any fixed $t\in[T]$,
\begin{equation}
    \E(\|\g_{t}^{(k)}-\nabla f_{t}(\x_{t}^{k})\|^{2})\le\frac{N_{0}}{(k+4)^{2/3}},
\end{equation} where $N_{0}=\max\{4^{2/3}\max_{t\in[T]}\|\nabla f_{t}(\x_{t}^{1})\|^{2},4\sigma^{2}+6(L_{0}r(\C))^{2}\}$.
\end{lemma}
\begin{proof}
According to \cref{alg:1}, 
 $\g_{t}^{(k)}=(1-\eta_{k})\g_{t}^{(k-1)}+\eta_{k}\widetilde{\nabla}f_{t}(\x_{t}^{(k)})$ where $\E(\widetilde{\nabla}f_{t}(\x_{t}^{(k)})|\x_{t}^{(k)})=\nabla f_{t}(\x_{t}^{(k)})$. We first derive that 
 \begin{equation}
     \|\nabla f_{t}(\x_{t}^{(k)})-\nabla f_{t}(\x_{t}^{(k-1)})\|\le\frac{L_{0}}{K}\|\vbf_{t}^{k}\|\le\frac{2L_{0}r(\C)}{k+3},
 \end{equation} where the first inequality follows from the $L_{0}$-smoothness of $f_{t}$
 Therefore, if we set the $\widetilde{\a}_{t}$ in \cref{lemma:vr1} as $\widetilde{\nabla}f_{t}(\x_{t}^{(k)})$, we have 
 \begin{equation}
    \E(\|\g_{t}^{(k)}-\nabla f_{t}(\x_{t}^{(k)})\|^{2})\le\frac{N_{0}}{(k+4)^{2/3}}.
\end{equation}
\end{proof}
Now, we present the proof of \cref{thm:1}.
\begin{proof}
  If we set $(f,\x_{k},\x_{k-1},\vbf_{k},\dbf,\y)$ in \cref{lemma:use} as  $(f_{t},\x^{(k+1)}_{t},\x^{(k)}_{t},\vbf^{(k+1)}_{t},\g_{t}^{(k)},\x^{*})$ in the \cref{alg:1} where $\x^{*}=\arg\max_{\x\in\C}\sum_{t=1}^{T}f_{t}(\x)$, we have, for any $k\in[K]$, 
 \begin{equation}
    \begin{aligned}
     &f_{t}(\x^{(k+1)}_{t})\\
     \ge & (1-\frac{1}{K})f_{t}(\x^{(k)}_{t})+\frac{1}{K}f_{t}(\x^{(k)}_{t}+(\one-\x^{(k)}_{t})\odot \x^{*})+\frac{1}{K}\langle (\one-\x^{(k)})\odot \g_{t}^{(k)},\vbf^{(k+1)}_{t}-\x^{*}\rangle\\
     &+\frac{1}{K}\langle(\vbf^{(k+1)}_{t}-\x^{*})\odot(\one-\x^{(k)}_{t}),\nabla f(\x_{k-1})-\g_{t}^{(k)}\rangle-\frac{L_{0}}{2}\|\x^{(k+1)}_{t}-\x^{(k)}_{t}\|^{2}\\
     \ge & (1-\frac{1}{K})f_{t}(\x^{(k)}_{t})+\frac{1}{K}(1-\frac{1}{K})^{k}f_{t}(\x^{*})+\frac{1}{K}\langle (\one-\x^{(k)})\odot \g_{t}^{(k)},\vbf^{(k+1)}_{t}-\x^{*}\rangle\\
     &+\frac{1}{K}\langle(\vbf^{(k+1)}_{t}-\x^{*})\odot(\one-\x^{(k)}_{t}),\nabla f(\x_{k-1})-\g_{t}^{(k)}\rangle-\frac{L_{0}}{2}\|\x^{(k+1)}_{t}-\x^{(k)}_{t}\|^{2},
    \end{aligned} 
\end{equation} where the second inequality comes from \cref{lemma:3} and \cref{lemma:4}.
Therefore, by iteration, we have
 \begin{equation}
    \begin{aligned}
     &f_{t}(\x^{(K)}_{t})\\
     \ge & (1-\frac{1}{K})f_{t}(\x^{(K-1)}_{t})+\frac{1}{K}(1-\frac{1}{K})^{K-1}f_{t}(\x^{*})+\frac{1}{K}\langle (\one-\x^{(K-1)})\odot \g_{t}^{(K-1)},\vbf^{(K)}_{t}-\x^{*}\rangle\\
     &+\frac{1}{K}\langle(\vbf^{(K)}_{t}-\x^{*})\odot(\one-\x^{(K-1)}_{t}),\nabla f(\x^{(K-1)}_{t})-\g_{t}^{(K-1)}\rangle-\frac{L_{0}r^{2}(\C)}{2K^{2}}\\
     \ge & \dots\\
     \ge & (1-\frac{1}{K})f_{t}(\x^{(0)}_{t})+(1-\frac{1}{K})^{K-1}f_{t}(\x^{*})+\frac{1}{K}\sum_{m=0}^{K-1}(1-\frac{1}{K})^{K-1-m}\langle (\one-\x^{(m)})\odot \g_{t}^{(m)},\vbf^{(m+1)}_{t}-\x^{*}\rangle\\
     &+\frac{1}{K}\sum_{m=0}^{K-1}(1-\frac{1}{K})^{K-1-m}\langle(\vbf^{(m+1)}_{t}-\x^{*})\odot(\one-\x^{(m)}_{t}),\nabla f(\x^{(m)}_{t})-\g_{t}^{(m)}\rangle-\frac{L_{0}r^{2}(\C)}{2K}\\
     \ge & (1-\frac{1}{K})f_{t}(\x^{(0)}_{t})+\frac{1}{e}f_{t}(\x^{*})+\frac{1}{K}\sum_{m=0}^{K-1}(1-\frac{1}{K})^{K-1-m}\langle (\one-\x^{(m)})\odot \g_{t}^{(m)},\vbf^{(m+1)}_{t}-\x^{*}\rangle\\
     &+\frac{1}{K}\sum_{m=0}^{K-1}(1-\frac{1}{K})^{K-1-m}\langle(\vbf^{(m+1)}_{t}-\x^{*})\odot(\one-\x^{(m)}_{t}),\nabla f(\x^{(m)}_{t})-\g_{t}^{(m)}\rangle-\frac{L_{0}r^{2}(\C)}{2K},
     \end{aligned} 
\end{equation} 
where the final inequality comes from $(1-\frac{1}{K})^{K-1}\ge\frac{1}{e}$.

Finally, 
\begin{equation}
    \begin{aligned}
     &\sum_{t=1}^{T}\E(f_{t}(\x^{(K)}_{t}))\\
     \ge & \frac{1}{e}\sum_{t=1}^{T}f_{t}(\x^{*})+\frac{1}{K}\sum_{t=1}^{T}\sum_{m=0}^{K-1}(1-\frac{1}{K})^{K-1-m}\E(\langle (\one-\x^{(m)})\odot \g_{t}^{(m)},\vbf^{(m+1)}_{t}-\x^{*}\rangle)\\
     &+\frac{1}{K}\sum_{t=1}^{T}\sum_{m=0}^{K-1}(1-\frac{1}{K})^{K-1-m}\E\langle(\vbf^{(m+1)}_{t}-\x^{*})\odot(\one-\x^{(m)}_{t}),\nabla f(\x^{(m)}_{t})-\g_{t}^{(m)}\rangle)-\frac{L_{0}Tr^{2}(\C)}{2K}\\
     = & \frac{1}{e}\sum_{t=1}^{T}f_{t}(\x^{*})+\frac{1}{K}\sum_{m=0}^{K-1}(1-\frac{1}{K})^{K-1-m}\sum_{t=1}^{T}\E(\langle (\one-\x^{(m)})\odot \g_{t}^{(m)},\vbf^{(m+1)}_{t}-\x^{*}\rangle)\\
     &+\frac{1}{K}\sum_{t=1}^{T}\sum_{m=0}^{K-1}(1-\frac{1}{K})^{K-1-m}\E(\langle(\vbf^{(m+1)}_{t}-\x^{*})\odot(\one-\x^{(m)}_{t}),\nabla f(\x^{(m)}_{t})-\g_{t}^{(m)}\rangle)-\frac{L_{0}Tr^{2}(\C)}{2K}\\
     \ge & \frac{1}{e}\sum_{t=1}^{T}f_{t}(\x^{*})-\frac{1}{K}\sum_{m=0}^{K-1}(1-\frac{1}{K})^{K-1-m}M_{0}\sqrt{T}-\frac{L_{0}Tr^{2}(\C)}{2K}\\
     &-\frac{1}{2K}\sum_{t=1}^{T}\sum_{m=0}^{K-1}(r^{2}(\C)b+\E(\frac{\|\nabla f(\x^{(m)}_{t})-\g_{t}^{(m)}\|^{2}}{b}))\\
     \ge & \frac{1}{e}\sum_{t=1}^{T}f_{t}(\x^{*})-M_{0}\sqrt{T}-\frac{L_{0}Tr^{2}(\C)}{2K}-r(\C)(\frac{3}{2}N_{0}+\frac{1}{2})\frac{T}{K^{1/3}},
     \end{aligned} 
\end{equation} 
where the second inequality follows from $\langle(\vbf^{(m+1)}_{t}-\x^{*})\odot(\one-\x^{(m)}_{t}),\nabla f(\x^{(m)}_{t})-\g_{t}^{(m)}\rangle\le\frac{1}{2}(\mathrm{diam}^{2}(\C)b+\frac{\|\nabla f(\x^{(m)}_{t})-\g_{t}^{(m)}\|^{2}}{b})$ for any $b>0$ and $\sum_{t=1}^{T}\E(\langle (\one-\x^{(m)})\odot \g_{t}^{(m)},\vbf^{(m+1)}_{t}-\x^{*}\rangle)\le M_{0}\sqrt{T}$ from Assumption~\ref{assumption1}; and the last inequality comes from $\sum_{m=0}^{K-1}(1-\frac{1}{K})^{K-1-m}\le K$, $\sum_{m=0}^{K-1}\E(\|\nabla f(\x^{(m)}_{t})-\g_{t}^{(m)}\|^{2})\le\sum_{m=0}^{K-1}\frac{N_{0}}{(m+4)^{2/3}}\le\int_{t=0}^{K}\frac{N_{0}}{t^{2/3}}\le3N_{0}K^{1/3}$ and setting $b=\frac{1}{\mathrm{diam}(\C)K^{1/3}}$.
 \end{proof}

\section{Proofs in \cref{sec:Mono}}\label{appendix:Mono}
In this section, we begin by deriving the upper bound of $\x_{q}^{(k)}$ in \cref{alg:2}. First, like the Equation~(\ref{equ:update1}), we also take a similar rule to update the $\x_{q}^{(k)}$. As a result, we have:
\begin{lemma}\label{lemma:upp2}
For $i\in[n]$ and $q\in[Q]$, we have $(\x_{q}^{(k)})_{i}\le1-(1-\frac{1}{K})^{k}$.
\end{lemma}
Before going into the detail, we define the average function of the remaining $(K-k)$ functions as $\bar{f}_{q,k}=\frac{\sum_{m=k+1}^{K}f_{t_{q}^{(m)}}}{K-k}$ for any $0\le k\le K-1$. Also, we use $\F_{q,k}$ to denote the $\sigma$-field generated via $t_{q}^{(1)},\dots,t_{q}^{(k)}$.
As a result, according to \cref{lemma:vr2} in variance reduction section, we show how $\g_{q}^{(k)}$~(See Line 13 in \cref{alg:2}) approximates the gradient $\nabla \bar{f}_{q,k-1}(\x_{q}^{k})$, i.e., 
\begin{lemma}\label{lemma:use1} 
Under Assumption~\ref{assumption1} and $\|\nabla f_{t}(\x)\|\le G$, if we set $\eta_{k}=\frac{2}{(k+3)^{2/3}}$, when $1\le k\le\frac{K}{2}+1$, and $\eta_{k}=\frac{1.5}{(K-k+2)^{2/3}}$, when $\frac{K}{2}+2\le k\le K$, we have, for any fixed $q\in[Q]$,
\begin{equation}
    \E(\|\g_{q}^{(k)}-\nabla \bar{f}_{q,k-1}(\x_{q}^{k})\|^{2})\le\left\{\begin{aligned}
       &\frac{N_{1}}{(k+4)^{2/3}}, & 1\le k\le\frac{K}{2}+1\\
       &\frac{N_{1}}{(K-k+1)^{2/3}}, & \frac{K}{2}+2\le k\le K
    \end{aligned}\right.
\end{equation} 
where $N_{1}=\max\{5^{2/3}G^{2},8(\sigma^{2}+G^{2})+32(2G+L_{0}r(\C))^{2},4.5(\sigma^{2}+G^{2})+7(2G+L_{0}r(\C))^{2}/3\}$.
\end{lemma}
\begin{proof}
From \cref{alg:2}, $\g_{q}^{(k)}=(1-\eta_{k})\g_{q}^{(k-1)}+\eta_{k}\widetilde{\nabla}f_{t_{q}^{(k)}}(\x_{q}^{(k)})$. As we know,  $\E(\widetilde{\nabla}f_{t_{q}^{(k)}}(\x_{q}^{(k)})|\F_{q,k-1})=\nabla\bar{f}_{q,k-1}(\x_{q}^{(k)})$. Also, we have
\begin{equation}
    \begin{aligned}
     &\nabla\bar{f}_{q,k-1}(\x_{q}^{(k)})-\nabla\bar{f}_{q,k-2}(\x_{q}^{(k-1)})\\
     = & \frac{\sum_{m=k}^{K}\nabla f_{t_{q}^{(m)}}(\x_{q}^{(k)})}{K-k+1}-\frac{\sum_{m=k-1}^{K}\nabla f_{t_{q}^{(m)}}(\x_{q}^{(k-1)})}{K-k+2}\\
     = & \frac{\sum_{m=k}^{K}(\nabla f_{t_{q}^{(m)}}(\x_{q}^{(k)})-\nabla f_{t_{q}^{(m)}}(\x_{q}^{(k-1)}))}{K-k+2}+\frac{\sum_{m=k}^{K}\nabla f_{t_{q}^{(m)}}(\x_{q}^{(k)})}{(K-k+1)(K-k+2)}-\frac{\nabla f_{t_{q}^{(k-1)}}(\x_{q}^{(k-1)})}{K-k+2}.
    \end{aligned}
\end{equation}
Thus, 
\begin{equation}
    \begin{aligned}
     &\|\nabla\bar{f}_{q,k-1}(\x_{q}^{(k)})-\nabla\bar{f}_{q,k-2}(\x_{q}^{(k-1)})\|\\
     \le & \|\frac{\sum_{m=k}^{K}(\nabla f_{t_{q}^{(m)}}(\x_{q}^{(k)})-\nabla f_{t_{q}^{(m)}}(\x_{q}^{(k-1)}))}{K-k+2}\|+\|\frac{\sum_{m=k}^{K}\nabla f_{t_{q}^{(m)}}(\x_{q}^{(k)})}{(K-k+1)(K-k+2)}\|+\|\frac{\nabla f_{t_{q}^{(k-1)}}(\x_{q}^{(k-1)})}{K-k+2}\|\\
     \le & \frac{(K-k+1)L_{0}r(\C)}{K(K-k+2)}+\frac{G}{K-k+2}+\frac{G}{K-k+2}\\
     \le & \frac{L_{0}r(\C)+2G}{K-k+2}.
    \end{aligned}
\end{equation}
Moreover,
\begin{equation}
    \begin{aligned}
     &\E(\|\widetilde{\nabla}f_{t_{q}^{(k)}}(\x_{q}^{(k)})-\nabla\bar{f}_{q,k-1}(\x_{q}^{(k)})\|^{2}|\F_{q,k-1})\\
     \le & 2(\E(\|\widetilde{\nabla}f_{t_{q}^{(k)}}(\x_{q}^{(k)})-\nabla f_{t_{q}^{(k)}}(\x_{q}^{(k)})\|^{2}|\F_{q,k-1})+\E(\|\nabla f_{t_{q}^{(k)}}(\x_{q}^{(k)})-\nabla\bar{f}_{q,k-1}(\x_{q}^{(k)})\|^{2}|\F_{q,k-1}))\\
     = & 2(\E(\|\widetilde{\nabla}f_{t_{q}^{(k)}}(\x_{q}^{(k)})-\nabla f_{t_{q}^{(k)}}(\x_{q}^{(k)})\|^{2}|\F_{q,k-1})+Var(\nabla f_{t_{q}^{(k)}}(\x_{q}^{(k)})|\F_{q,k-1}))\\
     \le & 2(\sigma^{2}+G^{2}),
    \end{aligned}
\end{equation} 
where $Var(\nabla f_{t_{q}^{(k)}}(\x_{q}^{(k)})|\F_{q,k-1})=\E(\|\nabla f_{t_{q}^{(k)}}(\x_{q}^{(k)})-\nabla\bar{f}_{q,k-1}(\x_{q}^{(k)})\|^{2}|\F_{q,k-1})$.

According to \cref{lemma:vr2} where we set $\widetilde{\a}_{k}=\widetilde{\nabla}f_{t_{q}^{(k)}}(\x_{q}^{(k)})$, we have
\begin{equation}
    \E(\|\g_{q}^{(k)}-\nabla \bar{f}_{q,k-1}(\x_{q}^{k})\|^{2})\le\left\{\begin{aligned}
       &\frac{N_{1}}{(k+4)^{2/3}}& 1\le k\le\frac{K}{2}+1\\
       &\frac{N_{1}}{(K-k+1)^{2/3}}& \frac{K}{2}+2\le k\le K
    \end{aligned}\right.
\end{equation} 
where $N_{1}=\max\{5^{2/3}G^{2},8(\sigma^{2}+G^{2})+32(2G+L_{0}r(\C))^{2},4.5(\sigma^{2}+G^{2})+7(2G+L_{0}r(\C))^{2}/3\}$.
\end{proof}
Now, we prove \cref{thm:2}.
\begin{proof}
Note that $\bar{f}_{q,k-1}$ is continuous DR-submodular and $L_{0}$-smooth. Thus, if we set $(f,\x_{k},\x_{k-1},\vbf_{k},\dbf,\y)$ in \cref{lemma:use} as  $(\bar{f}_{q,k-1},\x^{(k+1)}_{q},\x^{(k)}_{q},\vbf^{(k+1)}_{q},\g_{q}^{(k)},\x^{*})$ in the \cref{alg:2} where $\x^{*}=\arg\max_{\x\in\C}\sum_{t=1}^{T}f_{t}(\x)$, we have, for any $k\in[K]$,
\begin{equation}
    \begin{aligned}
     & \bar{f}_{q,k-1}(\x^{(k+1)}_{q})\\
     \ge & (1-\frac{1}{K})\bar{f}_{q,k-1}(\x^{(k)}_{q})+\frac{1}{K}\bar{f}_{q,k-1}(\x^{(k)}_{q}+(\one-\x^{(k)}_{q})\odot \x^{*})+\frac{1}{K}\langle (\one-\x^{(k)}_{q})\odot \g_{q}^{(k)},\vbf^{(k+1)}_{q}-\x^{*}\rangle\\
     &+\frac{1}{K}\langle(\vbf^{(k+1)}_{q}-\x^{*})\odot(\one-\x^{(k)}_{q}),\nabla \bar{f}_{q,k-1}(\x^{k}_{q})-\g_{q}^{(k)}\rangle-\frac{L_{0}}{2}\|\x^{(k+1)}_{q}-\x^{(k)}_{q}\|^{2}\\
     \ge & (1-\frac{1}{K})\bar{f}_{q,k-1}(\x^{(k)}_{q})+\frac{1}{K}(1-\frac{1}{K})^{k}\bar{f}_{q,k-1}(\x^{*})+\frac{1}{K}\langle (\one-\x^{(k)}_{q})\odot \g_{q}^{(k)},\vbf^{(k+1)}_{q}-\x^{*}\rangle\\
     &+\frac{1}{K}\langle(\vbf^{(k+1)}_{q}-\x^{*})\odot(\one-\x^{(k)}_{q}),\nabla\bar{f}_{q,k-1}(\x^{k}_{q})-\g_{q}^{(k)}\rangle-\frac{L_{0}}{2}\|\x^{(k+1)}_{q}-\x^{(k)}_{q}\|^{2},
    \end{aligned} 
\end{equation} 
where the second inequality comes from \cref{lemma:upp2} and \cref{lemma:4}.
Therefore, by iteration, we have
\begin{equation}
    \begin{aligned}
    &\E(\bar{f}_{q}(\x^{(K)}_{q}))\\
    = & \E(\bar{f}_{q,K-2}(\x^{(K)}_{q}))\\
    \ge & (1-\frac{1}{K})\E(\bar{f}_{q,K-2}(\x^{(K-1)}_{q}))+\frac{1}{K}(1-\frac{1}{K})^{K-1}\E(\bar{f}_{q,K-2}(\x^{*}))+\frac{1}{K}\E(\langle (\one-\x^{(K-1)}_{q})\odot \g_{q}^{(K-1)},\vbf^{(K)}_{q}-\x^{*}\rangle)\\
    &+\frac{1}{K}\E(\langle(\vbf^{(K)}_{q}-\x^{*})\odot(\one-\x^{(K-1)}_{q}),\nabla\bar{f}_{q,K-2}(\x^{K-1}_{q})-\g_{q}^{(K-1)}\rangle)-\frac{L_{0}r^{2}(\C)}{2K^{2}}\\
    = & (1-\frac{1}{K})\E(\bar{f}_{q,K-3}(\x^{(K-1)}_{q}))+\frac{1}{K}(1-\frac{1}{K})^{K-1}\bar{f}_{q}(\x^{*})+\frac{1}{K}\E(\langle (\one-\x^{(K-1)}_{q})\odot \g_{q}^{(K-1)},\vbf^{(K)}_{q}-\x^{*}\rangle)\\
    &+\frac{1}{K}\E(\langle(\vbf^{(K)}_{q}-\x^{*})\odot(\one-\x^{(K-1)}_{q}),\nabla\bar{f}_{q,K-2}(\x^{K-1}_{q})-\g_{q}^{(K-1)}\rangle)-\frac{L_{0}r^{2}(\C)}{2K^{2}}\\
    \ge & \dots\\
    \ge & (1-\frac{1}{K})^{K-1}f_{q}(\x^{*})+\frac{1}{K}\sum_{m=1}^{K-1}(1-\frac{1}{K})^{K-1-m}\E(\langle (\one-\x^{(m)}_{q})\odot \g_{q}^{(m)},\vbf^{(m+1)}_{q}-\x^{*}\rangle)\\
    &+\frac{1}{K}\sum_{m=1}^{K-1}(1-\frac{1}{K})^{K-1-m}\E(\langle(\vbf^{(m+1)}_{q}-\x^{*})\odot(\one-\x^{(m)}_{q}),\nabla\bar{f}_{q,m-1}(\x^{m}_{q})-\g_{q}^{(m)}\rangle)-\frac{L_{0}r^{2}(\C)}{2K}.
    \end{aligned} 
\end{equation} 
Finally, 
\begin{equation}
    \begin{aligned}
     &\sum_{q=1}^{Q}\E(\bar{f}_{q}(\x^{(K)}_{q}))\\
     \ge & (1-\frac{1}{K})^{K-1}\sum_{q=1}^{Q}f_{q}(\x^{*})+\frac{1}{K}\sum_{q=1}^{Q}\sum_{m=1}^{K-1}(1-\frac{1}{K})^{K-1-m}\E(\langle (\one-\x^{(m)}_{q})\odot \g_{q}^{(m)},\vbf^{(m+1)}_{q}-\x^{*}\rangle)\\
     &+\frac{1}{K}\sum_{q=1}^{Q}\sum_{m=1}^{K-1}(1-\frac{1}{K})^{K-1-m}\E(\langle(\vbf^{(m+1)}_{q}-\x^{*})\odot(\one-\x^{(m)}_{q}),\nabla\bar{f}_{q,m-1}(\x^{m}_{q})-\g_{q}^{(m)}\rangle)-\frac{L_{0}Qr^{2}(\C)}{2K}\\
     \ge & \frac{1}{e}\sum_{q=1}^{Q}f_{q}(\x^{*})+\frac{1}{K}\sum_{m=1}^{K-1}(1-\frac{1}{K})^{K-1-m}\sum_{q=1}^{Q}\E(\langle (\one-\x^{(m)}_{q})\odot \g_{q}^{(m)},\vbf^{(m+1)}_{q}-\x^{*}\rangle)\\
     &-\frac{1}{2K}\sum_{q=1}^{Q}\sum_{m=1}^{K-1}\E(b_{m}*\mathrm{diam}^{2}(\C)+\frac{\|\nabla\bar{f}_{q,m-1}(\x^{m}_{q})-\g_{q}^{(m)}\|^{2}}{b_{m}})-\frac{L_{0}Qr^{2}(\C)}{2K}\\
     \ge & \frac{1}{e}\sum_{q=1}^{Q}f_{q}(\x^{*})-M_{0}\sqrt{Q}-\frac{L_{0}Qr^{2}(\C)}{2K}-\frac{1}{2K}\sum_{q=1}^{Q}\sum_{m=1}^{K-1}\E(b_{m}*\mathrm{diam}^{2}(\C)+\frac{\|\nabla\bar{f}_{q,m-1}(\x^{m}_{q})-\g_{q}^{(m)}\|^{2}}{b_{m}}),
     \end{aligned}
\end{equation} 
where the second inequality comes from $(1-\frac{1}{K})^{K-1}\ge\frac{1}{e}$ and $\langle(\vbf^{(m+1)}_{q}-\x^{*})\odot(\one-\x^{(m)}_{q}),\nabla\bar{f}_{q,m-1}(\x^{m}_{q})-\g_{q}^{(m)}\rangle\le\frac{1}{2}(b_{m}*\mathrm{diam}^{2}(\C)+\frac{\|\nabla\bar{f}_{q,m-1}(\x^{m}_{q})-\g_{q}^{(m)}\|^{2}}{b_{m}})$ for any positive constant $b_{m}>0$; the third comes from $\sum_{q=1}^{Q}\E(\langle (\one-\x^{(m)}_{q})\odot \g_{q}^{(m)},\vbf^{(K)}_{q}-\x^{*}\rangle)\le M_{0}\sqrt{Q}$.

If we consider $b_{m}=\frac{1}{\mathrm{diam}(\C)(m+4)^{1/3}}$ when $1\le m\le\frac{K}{2}+1$ and $b_{m}=\frac{1}{\mathrm{diam}(\C)(K-m+1)^{1/3}}$ when $\frac{K}{2}+2\le m\le K$, then we have 
\begin{equation}
    \begin{aligned}
    &\sum_{m=1}^{K-1}\mathrm{diam}^{2}(\C)b_{m}\le\sum_{m=1}^{K/2+1}\frac{\mathrm{diam}(\C)}{(m+4)^{1/3}}+\sum_{m=K/2+2}^{K}\frac{\mathrm{diam}(\C)}{(K-m+1)^{1/3}}\le2\mathrm{diam}(\C)K^{2/3},\\
    &\sum_{m=1}^{K-1}\E(\frac{\|\nabla\bar{f}_{q,m-1}(\x^{m}_{q})-\g_{q}^{(m)}\|^{2}}{b_{m}})\le\sum_{m=1}^{K/2+1}\frac{\mathrm{diam}(\C)N_{1}}{(m+4)^{1/3}}+\sum_{m=K/2+2}^{K}\frac{N_{1}}{(K-m+1)^{1/3}}\le2N_{1}\mathrm{diam}(\C)K^{2/3},
    \end{aligned}
\end{equation} 
where the second inequality comes from \cref{lemma:use1} and $N_{1}=\max\{5^{2/3}G^{2},8(\sigma^{2}+G^{2})+32(2G+L_{0}r(\C))^{2},4.5(\sigma^{2}+G^{2})+7(2G+L_{0}r(\C))^{2}/3\}$.

As a result, 
\begin{equation}
    \begin{aligned}
     &\frac{1}{e}\sum_{t=1}^{T}f_{t}(\x^{*})-\sum_{t=1}^{T}\E(f_{t}(\y_{t}))\\
     = & K(\frac{1}{e}\sum_{q=1}^{Q}\bar{f}_{q}(\x^{*})-\sum_{q=1}^{Q}\E(\bar{f}_{q}(\x_{q}^{(K)})))\\
     \le & 2\mathrm{diam}(\C)(N_{1}+1)QK^{2/3}+\frac{L_{0}r^{2}(\C)}{2}Q+M_{0}\sqrt{Q}K.
    \end{aligned}
\end{equation}
\end{proof}

\section{Proofs in \cref{sec:Bandit}}\label{appendix:Bandit}
To begin, we review the properties of smoothed function.
\begin{lemma}[\citet{zhang2019online,chen2020black}] 
If $f:[0,1]^{n}\rightarrow\R_{+}$ is continuous DR-submodular, $G$-Lipschitz, and $L_{0}$-smooth, then so is $\hat{f}_{\delta}$ where $\hat{f}_{\delta}(\x)=\E_{\vbf\sim B^{n}}(f(\x+\delta\vbf))$ and  we have $|\hat{f}_{\delta}(\x)-f(\x)|\le G\delta$ for all $\x\in[0,1]^{n}$.
\end{lemma}
In this section, we begin by examining the sequence of iterates $\x^{(0)}_{q},\x^{(1)}_{q},\dots,\x^{(K)}_{q}$ in \cref{alg:3}. First, we derive the upper bound of $\tilde{\x}_{q}^{(k)}$ where $\tilde{\x}_{q}^{(k)}=(\x^{(k)}_{q}-\delta\one)\oslash(\one-\delta\one)$.
\begin{lemma}\label{lemma:3.1}
For $i\in[n]$ and $q\in[Q]$, we have $(\tilde{\x}_{q}^{(k)})_{i}\le1-(1-\frac{1}{K})^{k}$ where $\tilde{\x}_{q}^{(k)}=(\x^{(k)}_{q}-\delta\one)\oslash(\one-\delta\one)$.
\end{lemma}
\begin{proof}
From \cref{alg:3} (See Line 7), we have
\begin{equation}
    \begin{aligned}
     \x_{q}^{(k)}=\x_{q}^{(k-1)}+\frac{1}{K}\tilde{\vbf}_{q}^{(k)}\odot(\one-\x_{q}^{(k-1)}).  
    \end{aligned}
\end{equation}
Therefore, we have
\begin{equation}
    \begin{aligned}
     \tilde{\x}_{q}^{(k)}=\tilde{\x}_{q}^{(k-1)}+\frac{1}{K}\tilde{v}_{q}^{(k)}\odot(\one-\tilde{\x}_{q}^{(k-1)}).
    \end{aligned}
\end{equation}
Finally, due to $0\le\tilde{v}_{q}^{(k)}\le\mathbf{1}$, we obtain
\begin{equation}
    \begin{aligned}
     \tilde{\x}_{q}^{(k)}\le\tilde{\x}_{q}^{(k-1)}+\frac{1}{K}(\mathbf{1}-\tilde{\x}_{q}^{(k-1)}).
    \end{aligned}
\end{equation}
According to \cref{lemma:3}, we get the result.
\end{proof}
Next, we define some notations frequently used in this section. For any $f_{t}$, we denote its $\delta$-smoothed approximation as $\hat{f}_{t,\delta}(\x)=\E_{\vbf\sim B^{n}}(f_{t}(\x+\delta\vbf))$. Then, the average function for $q$-block is denoted as $\bar{F}_{q}(\x)=\frac{\sum_{m=(q-1)L+1}^{qL}\hat{f}_{m,\delta}((\one-\delta\one)\odot \x+\delta\one)}{L}$. Also, the average function of remaining $(L-l)$ rewards is 
$\bar{F}_{q,l}(\x)=\frac{\sum_{m=l+1}^{L}\hat{f}_{t_{q}^{(m)},\delta}(((\one-\delta\one)\odot\x+\delta\one))}{L-l}$ where $0\le l\le L-1$. 

In the following part, we assume $\x^{*}=\arg\max_{\x\in\C}f_{t}(\x)$ and $\x^{*}_{\delta}=\arg\max_{\x\in\C^{'}}f_{t}(\x)$. Then, we could conclude that
\begin{lemma}\label{lemma:3+}
Under Assumption~\ref{assumption2}, if $\|\nabla f_{t}(\x)\|\le G$, then
\begin{equation} 
    \begin{aligned}
    & \sum_{t=1}^{T}\frac{1}{e}f_{t}(\x^{*})-\sum_{t=1}^{T}f_{t}(\y_{t}) \\
    \le & L\sum_{q=1}^{Q}\frac{1}{e}\bar{F}_{q}(\tilde{\x}^{*}_{\delta})-L\sum_{q=1}^{Q}\bar{F}_{q}(\tilde{\x}_{q}^{(K)})+2M_{1}KQ+\left((\sqrt{n}+1)\frac{r(\C)}{r}+\sqrt{n}+2\right)TG\delta,
    \end{aligned}
\end{equation} 
where $\tilde{\x}^{*}_{\delta}=(\x^{*}_{\delta}-\delta\one)\oslash(\one-\delta\one)$ and $\tilde{\x}^{(K)}_{q}=(\x^{(K)}_{q}-\delta\one)\oslash(\one-\delta\one)$. 
\end{lemma}
\begin{proof}
We denote the $\x^{'}$ as the projection of $\x^{*}$ on the $\C^{'}$, i.e., $\x^{'}=\arg\min_{x\in\C^{'}}\|\x-\x^{*}\|$, we could conclude that 
\begin{equation}
    \begin{aligned}
     \sum_{t=1}^{T}\frac{1}{e}f_{t}(\x^{*})-\sum_{t=1}^{T}f_{t}(\y_{t})= & \sum_{t=1}^{T}\frac{1}{e}f_{t}(\x^{*})-\sum_{t=1}^{T}\frac{1}{e}f_{t}(\x^{*}_{\delta})+\sum_{t=1}^{T}\frac{1}{e}f_{t}(\x^{*}_{\delta})-\sum_{t=1}^{T}\frac{1}{e}\hat{f}_{t,\delta}(\x^{*}_{\delta})\\
     & +\sum_{t=1}^{T}\frac{1}{e}\hat{f}_{t,\delta}(\x^{*}_{\delta})-\sum_{t=1}^{T}\hat{f}_{t,\delta}(\y_{t})+\sum_{t=1}^{T}\hat{f}_{t,\delta}(\y_{t})-\sum_{t=1}^{T}f_{t}(\y_{t}).
\end{aligned}
\end{equation}
First, $|\hat{f}_{t,\delta}(\y_{t})-f_{t}(\y_{t})|\le G\delta$ and $|\hat{f}_{t,\delta}(\x^{*}_{\delta})-f_{t}(\x^{*}_{\delta})|\le G\delta$.
Then, 
\begin{equation}
    \begin{aligned}
    &\sum_{t=1}^{T}f_{t}(\x^{*})-\sum_{t=1}^{T}f_{t}(\x^{*}_{\delta})\\
    \le &\sum_{t=1}^{T}f_{t}(\x^{*})-\sum_{t=1}^{T}f_{t}(\x^{'})\\
    \le & TG\|\x^{*}-\x^{'}\|\\
    \le & \left((\sqrt{n}+1)\frac{r(\C)}{r}+\sqrt{n}\right)TG\delta,
    \end{aligned} 
\end{equation} where the first inequality comes from the definition of $\x^{*}_{\delta}$ and $\x^{'}\in\C'$; the second follows from the lipschitz of $f_{t}$; the final from \cref{lemma:construct_set}.

Finally, if setting $\tilde{\x}^{*}_{\delta}=(\x^{*}_{\delta}-\delta\one)\oslash(\one-\delta\one)$ and $\tilde{\x}^{(K)}_{q}=(\x^{(K)}_{q}-\delta\one)\oslash(\one-\delta\one)$, 
\begin{equation}
    \begin{aligned}
    &\sum_{t=1}^{T}\frac{1}{e}\hat{f}_{t,\delta}(\x^{*}_{\delta})-\sum_{t=1}^{T}\hat{f}_{t,\delta}(\y_{t})\\
    = & L\sum_{q=1}^{Q}\frac{1}{e}\bar{F}_{q}(\tilde{\x}^{*}_{\delta})-L\sum_{q=1}^{Q}\bar{F}_{q}(\tilde{\x}_{q}^{(K)})+\sum_{q=1}^{Q}\sum_{k=1}^{K}(\hat{f}_{t^{(k)}_{q}}(\x_{q}^{(K)})-\hat{f}_{t^{(k)}_{q}}(\y_{t^{(k)}_{q}}))\\
    \le & L\sum_{q=1}^{Q}\frac{1}{e}\bar{F}_{q}(\tilde{\x}^{*}_{\delta})-L\sum_{q=1}^{Q}\bar{F}_{q}(\tilde{\x}_{q}^{(K)})+2M_{1}KQ,
    \end{aligned}
\end{equation}
where the inequality comes from $|\hat{f}_{t^{k}_{q}}(\x_{q}^{(K)})-\hat{f}_{t^{k}_{q}}(\y_{t^{k}_{q}})|\le2M_{1}$.
Therefore,
\begin{equation} 
    \begin{aligned}
    & \sum_{t=1}^{T}\frac{1}{e}f_{t}(\x^{*})-\sum_{t=1}^{T}f_{t}(\y_{t}) \\
    \le & L\sum_{q=1}^{Q}\frac{1}{e}\bar{F}_{q}(\tilde{\x}^{*}_{\delta})-L\sum_{q=1}^{Q}\bar{F}_{q}(\tilde{\x}_{q}^{(K)})+2M_{1}KQ+\left((\sqrt{n}+1)\frac{r(\C)}{r}+\sqrt{n}+2\right)TG\delta.
    \end{aligned}
\end{equation}
\end{proof}
Next, we demonstrate how $(\one-\delta\one)\odot\g_{q}^{(k)}$(See Line 19 in \cref{alg:3}) approximates $\nabla \bar{F}_{q,k-1}(\tilde{\x}_{q}^{(k)})$ where $\tilde{\x}_{q}^{(k)}=(\x_{q}^{(k)}-\delta\one)\oslash(\one-\delta\one)$.
\begin{lemma}\label{lemma:vr3}
Under Assumption~\ref{assumption1} and Assumption~\ref{assumption2}, if $\|\nabla f_{t}(\x)\|\le G$, $L\ge2K$ and $\eta_{k}=\frac{2}{(k+3)^{2/3}}$ for $k\in[K]$, we have, for any fixed $q\in[Q]$, 
\begin{equation}
    \E(\|(\one-\delta\one)\odot\g_{q}^{(k)}-\nabla \bar{F}_{q,k-1}(\tilde{\x}_{q}^{(k)})\|^{2})\le\frac{N_{2}}{(k+4)^{2/3}},
\end{equation} where $N_{2}=\max\{3^{2/3}G^{2},8(\frac{n^{2}M_{1}^{2}}{\delta^{2}}+G^{2})+3(4.5L_{0}r(\C)+3G)^{2}/2\}$.
\end{lemma}
\begin{proof}
Similarly, we use $\F_{q,k}$ to denote the $\sigma$-field generated via $t_{q}^{(1)},\dots,t_{q}^{(k)}$. From Algorithm~\ref{alg:3}, $\g_{q}^{(k)}=(1-\eta_{k})\g_{q}^{(k-1)}+\eta_{k}\frac{n}{\delta}f_{t^{(k)}_{q}}(\x_{q}^{(k)}+\delta\ubf_{q}^{(k)})\ubf_{q}^{(k)}$. Also, we have  $\frac{n}{\delta}\E(f_{t_{q}^{(k)}}(\x_{q}^{(k)}+\delta\ubf_{q}^{(k)})(\one-\delta\one)\odot\ubf_{q}^{(k)}|\F_{q,k-1})=\nabla\bar{F}_{q,k-1}(\tilde{\x}_{q}^{(k)})$. Next, we prove that
\begin{equation}
    \begin{aligned}
     &\nabla\bar{F}_{q,k-1}(\tilde{\x}_{q}^{(k)})-\nabla\bar{F}_{q,k-2}(\tilde{\x}_{q}^{(k-1)})\\
     = &\frac{\sum_{m=k}^{L}(1-\delta)\nabla \hat{f}_{t_{q}^{(m)},\delta}(\x_{q}^{(k)})}{L-k+1}-\frac{\sum_{m=k-1}^{L}(1-\delta)\nabla \hat{f}_{t_{q}^{(m)},\delta}(\x_{q}^{(k-1)})}{L-k+2}\\
     = & \frac{\sum_{m=k}^{L}(1-\delta)(\nabla \hat{f}_{t_{q}^{(m)},\delta}(\x_{q}^{(k)})-\nabla \hat{f}_{t_{q}^{(m)},\delta}(\x_{q}^{(k-1)}))}{L-k+2}+\frac{\sum_{m=k}^{L}(1-\delta)\nabla \hat{f}_{t_{q}^{(m)},\delta}(\x_{q}^{(k)})}{(L-k+1)(L-k+2)}\\
     & -\frac{(1-\delta)\nabla\hat{f}_{t_{q}^{(k-1)},\delta}(\x_{q}^{(k-1)})}{L-k+2}.
    \end{aligned}
\end{equation}
Thus, when $L\ge2K$, 
\begin{equation}
    \begin{aligned}
     &\|\nabla\bar{F}_{q,k-1}(\tilde{\x}_{q}^{(k)})-\nabla\bar{F}_{q,k-2}(\tilde{\x}_{q}^{(k-1)})\|\\
     \le & \|\frac{\sum_{m=k}^{L}(\nabla \hat{f}_{t_{q}^{(m)},\delta}(\x_{q}^{(k)})-\nabla \hat{f}_{t_{q}^{(m)},\delta}(\x_{q}^{(k-1)}))}{L-k+2}\|+\|\frac{\sum_{m=k}^{L}\nabla \hat{f}_{t_{q}^{(m)},\delta}(\x_{q}^{(k)})}{(L-k+1)(L-k+2)}\|+\|\frac{\nabla\hat{f}_{t_{q}^{(k-1)},\delta}(\x_{q}^{(k-1)})}{L-k+2}\|\\
     \le & \frac{(L-k+1)L_{0}r(\C)}{K(L-k+2)}+\frac{G}{L-k+2}+\frac{G}{L-k+2}\\
     \le & \frac{4.5L_{0}r(\C)+3G}{k+3},
    \end{aligned}
\end{equation} where the final inequality follows from $L-k+2\ge2K-k+2\ge k+2$.
Moreover,
\begin{equation}
    \begin{aligned}
     &\E(\|\frac{n}{\delta}f_{t_{q}^{(k)}}(\x_{q}^{(k)}+\delta\ubf_{q}^{(k)})(\one-\delta\one)\odot\ubf_{q}^{(k)}-\nabla\bar{F}_{q,k-1}(\tilde{\x}_{q}^{(k)})\|^{2}|\F_{q,k-1})\\
     \le & 2\E(\|\frac{n}{\delta}f_{t_{q}^{(k)}}(\x_{q}^{(k)}+\delta\ubf_{q}^{(k)})(\one-\delta\one)\odot\ubf_{q}^{(k)}-(1-\delta)\one\odot\nabla\hat{f}_{t_{q}^{(k)},\delta}(\x_{q}^{(k)})\|^{2}|\F_{q,k-1})\\
     &+2\E((1-\delta)\one\odot\nabla\hat{f}_{t_{q}^{(k)},\delta}(\x_{q}^{(k)})-\nabla\bar{F}_{q,k-1}(\tilde{\x}_{q}^{(k)})\|^{2}|\F_{q,k-1})\\
     = & 2\E(\|\frac{n}{\delta}f_{t_{q}^{(k)}}(\x_{q}^{(k)}+\delta\ubf_{q}^{(k)})(\one-\delta\one)\odot\ubf_{q}^{(k)}-(1-\delta)\one\odot\nabla\hat{f}_{t_{q}^{(k)},\delta}(\x_{q}^{(k)})\|^{2}|\F_{q,k-1})\\
     &+2Var((1-\delta)\one\odot\nabla\hat{f}_{t_{q}^{(k)},\delta}(\x_{q}^{(k)})|\F_{q,k-1}))\\
     \le & 2\left(\frac{n^{2}M_{1}^{2}}{\delta^{2}}+G^{2}\right).
    \end{aligned}
\end{equation} 
According to \cref{lemma:vr2} where we set $\widetilde{\a}_{k}=\frac{n}{\delta}f_{t_{q}^{(k)}}(\x_{q}^{(k)}+\delta\ubf_{q}^{(k)})(\one-\delta\one)\odot\ubf_{q}^{(k)}$, we have
\begin{equation}
    \E(\|(\one-\delta\one)\odot\g_{q}^{(k)}-\nabla F_{q,k-1}(\tilde{\x}_{q}^{(k)})\|^{2})\le\frac{N_{2}}{(k+4)^{2/3}},
\end{equation} where $N_{2}=\max\{3^{2/3}G^{2},8(\frac{n^{2}M_{1}^{2}}{\delta^{2}}+G^{2})+3(4.5L_{0}r(\C)+3G)^{2}/2\}$.
\end{proof}

\begin{lemma}\label{lemma:final1}
Under Assumption~\ref{assumption1} and Assumption~\ref{assumption2}, if $\|\widetilde{\nabla}f_{t}(\x)\|\le G$ and $L\ge2K$, we could conclude that
\begin{equation}
    \sum_{q=1}^{Q}\frac{1}{e}\bar{F}_{q}(\x^{*}_{\delta})-\sum_{q=1}^{Q}\E(\bar{F}_{q}(\tilde{\x}^{(K)}_{q}))\le\frac{L_{0}Qr^{2}(\C)}{2K}+M_{0}\sqrt{Q}+\frac{\mathrm{diam}(\C)Q}{2\delta K^{1/3}}+\frac{\mathrm{diam}(\C)N_{2}\delta Q}{2K^{1/3}}.
\end{equation}
 \end{lemma}
\begin{proof}
If we set $(f,\x_{k},\x_{k-1},\vbf_{k})$ in \cref{lemma:use} as  $(\bar{F}_{q,k-1},\tilde{\x}^{(k+1)}_{q},\tilde{\x}^{(k)}_{q},\vbf^{(k+1)}_{q})$ in the \cref{alg:3}, we have
 \begin{equation}
    \begin{aligned}
     &\bar{F}_{q,k-1}(\tilde{\x}^{(k+1)}_{q}) \\ 
     \ge & (1-\frac{1}{K})\bar{F}_{q,k-1}(\tilde{\x}^{(k)}_{q})+\frac{1}{K}\bar{F}_{q,k-1}(\tilde{\x}^{(k)}_{q}+(\one-\tilde{\x}^{(k)}_{q})\odot\y)-\frac{L_{0}r(\C)^{2}}{2K^{2}} \\
     &+\frac{1}{K}\langle (\one-\tilde{\x}^{(k)}_{q})\odot d,\tilde{\vbf}^{(k+1)}_{q}-\y\rangle+\frac{1}{K}\langle(\tilde{\vbf}^{(k+1)}_{q}-\tilde{\x}^{*}_{\delta})\odot(\one-\tilde{\x}^{(k)}_{q}),\nabla \bar{F}_{q,k-1}(\tilde{\x}^{k}_{q})-\dbf\rangle \\
     \ge & (1-\frac{1}{K})\bar{F}_{q,k-1}(\tilde{\x}^{(k)}_{q})+\frac{1}{K}(1-\frac{1}{K})^{k}\bar{F}_{q,k-1}(\y)+\frac{1}{K}\langle (\one-\tilde{\x}^{(k)}_{q})\odot \dbf,\tilde{\vbf}^{(k+1)}_{q}-\tilde{\x}^{*}_{\delta}\rangle \\
     &+\frac{1}{K}\langle(\tilde{\vbf}^{(k+1)}_{q}-\y)\odot(\one-\tilde{\x}^{(k)}_{q}),\nabla\bar{F}_{q,k-1}(\tilde{\x}^{k}_{q})-\dbf\rangle-\frac{L_{0}r(\C)^{2}}{2K^{2}},
    \end{aligned} 
\end{equation} where the second inequality comes from \cref{lemma:3.1} and \cref{lemma:4}.

If we set $\dbf=(1-\delta)\g_{q}^{(k)}$ and $\y=\tilde{\x}^{*}_{\delta}$, we have
 \begin{equation}
    \begin{aligned}
    &\bar{F}_{q,k-1}(\tilde{\x}^{(k+1)}_{q})\\
    \ge & (1-\frac{1}{K})\bar{F}_{q,k-1}(\tilde{\x}^{(k)}_{q})+\frac{1}{K}(1-\frac{1}{K})^{k}\bar{F}_{q,k-1}(\tilde{\x}^{*}_{\delta})+\frac{1}{K}\langle (\one-\tilde{\x}^{(k)}_{q})\odot ((1-\delta)\g_{q}^{(k)}),\tilde{\vbf}^{(k+1)}_{q}-\tilde{\x}^{*}_{\delta}\rangle\\
    &+\frac{1}{K}\langle(\tilde{\vbf}^{(k+1)}_{q}-\tilde{\x}^{*}_{\delta})\odot(\one-\tilde{\x}^{(k)}_{q}),\nabla\bar{F}_{q,k-1}(\tilde{\x}^{k}_{q})-(1-\delta)\g_{q}^{(k)}\rangle-\frac{L_{0}r^{2}(\C)}{2K^{2}}\\
    = & (1-\frac{1}{K})\bar{F}_{q,k-1}(\tilde{\x}^{(k)}_{q})+\frac{1}{K}(1-\frac{1}{K})^{k}\bar{F}_{q,k-1}(\tilde{\x}^{*}_{\delta})+\frac{1}{K}\langle (\one-\tilde{\x}^{(k)}_{q})\odot\g_{q}^{(k)},(1-\delta)(\tilde{\vbf}^{(k+1)}_{q}-\tilde{\x}^{*}_{\delta})\rangle\\
    &+\frac{1}{K}\langle(\tilde{\vbf}^{(k+1)}_{q}-\tilde{\x}^{*}_{\delta})\odot(\one-\tilde{\x}^{(k)}_{q}),\nabla\bar{F}_{q,k-1}(\tilde{\x}^{k}_{q})-(1-\delta)\g_{q}^{(k)}\rangle-\frac{L_{0}r^{2}(\C)}{2K^{2}}\\
    = & (1-\frac{1}{K})\bar{F}_{q,k-1}(\tilde{\x}^{(k)}_{q})+\frac{1}{K}(1-\frac{1}{K})^{k}\bar{F}_{q,k-1}(\tilde{\x}^{*}_{\delta})+\frac{1}{K}\langle (\one-\tilde{\x}^{(k)}_{q})\odot\g_{q}^{(k)},\vbf^{(k+1)}_{q}-\x^{*}_{\delta})\rangle\\
    &+\frac{1}{K}\langle(\tilde{\vbf}^{(k+1)}_{q}-\tilde{\x}^{*}_{\delta})\odot(\one-\tilde{\x}^{(k)}_{q}),\nabla\bar{F}_{q,k-1}(\tilde{\x}^{k}_{q})-(1-\delta)\g_{q}^{(k)}\rangle-\frac{L_{0}r^{2}(\C)}{2K^{2}}.
    \end{aligned} 
\end{equation} 
Therefore, by iteration, we have
 \begin{equation}
    \begin{aligned}
     &\E(\bar{F}_{q}(\tilde{\x}^{(K)}_{q}))\\
     = & \E(\bar{F}_{q,K-2}(\tilde{\x}^{(K)}_{q}))\\
     \ge & (1-\frac{1}{K})\E(\bar{F}_{q,K-2}(\tilde{\x}^{(K-1)}_{q}))+\frac{1}{K}(1-\frac{1}{K})^{K-1}\E(\bar{F}_{q,K-2}(\tilde{\x}^{*}_{\delta}))+\frac{1}{K}\E(\langle (\one-\tilde{\x}^{(K-1)}_{q})\odot \g_{q}^{(K-1)},\vbf^{(K)}_{q}-\x^{*}_{\delta})\rangle)\\
     &+\frac{1}{K}\E(\langle(\tilde{\vbf}^{(K)}_{q}-\tilde{\x}^{*}_{\delta})\odot(\one-\tilde{\x}^{(K-1)}_{q}),\nabla\bar{F}_{q,K-2}(\tilde{\x}^{K-1}_{q})-(1-\delta)\g_{q}^{(K-1)}\rangle)-\frac{L_{0}r^{2}(\C)}{2K^{2}}\\
     = & (1-\frac{1}{K})\E(\bar{F}_{q,K-3}(\tilde{\x}^{(K-1)}_{q}))+\frac{1}{K}(1-\frac{1}{K})^{K-1}\bar{F}_{q}(\tilde{\x}^{*}_{\delta})+\frac{1}{K}\E(\langle (\one-\tilde{\x}^{(K-1)}_{q})\odot \g_{q}^{(K-1)},\vbf^{(K)}_{q}-\x^{*}_{\delta})\rangle)\\
     &+\frac{1}{K}\E(\langle(\tilde{\vbf}^{(K)}_{q}-\tilde{\x}^{*}_{\delta})\odot(\one-\tilde{\x}^{(K-1)}_{q}),\nabla\bar{F}_{q,K-2}(\tilde{\x}^{K-1}_{q})-(1-\delta)\g_{q}^{(K-1)}\rangle)-\frac{L_{0}r^{2}(\C)}{2K^{2}}\\
     \ge & \dots\\
     \ge & (1-\frac{1}{K})^{K-1}\bar{F}_{q}(\x^{*}_{\delta})+\frac{1}{K}\sum_{m=1}^{K-1}(1-\frac{1}{K})^{K-1-m}\E(\langle (\one-\tilde{\x}^{(m)}_{q})\odot \g_{q}^{(m)},\vbf^{(m+1)}_{q}-\x^{*}_{\delta}\rangle)\\
     &+\frac{1}{K}\sum_{m=1}^{K-1}(1-\frac{1}{K})^{K-1-m}\E(\langle(\tilde{\vbf}^{(m+1)}_{q}-\tilde{\x}^{*}_{\delta})\odot(\one-\tilde{\x}^{(m)}_{q}),\nabla\bar{F}_{q,m-1}(\tilde{\x}^{m}_{q})-(1-\delta)\g_{q}^{(m)}\rangle)-\frac{L_{0}r^{2}(\C)}{2K}.
     \end{aligned} 
\end{equation} 
Then, 
\begin{equation}
    \begin{aligned}
    &\sum_{q=1}^{Q}\E(\bar{F}_{q}(\tilde{\x}^{(K)}_{q}))\\
    \ge & (1-\frac{1}{K})^{K-1}\sum_{q=1}^{Q}\bar{F}_{q}(\x^{*}_{\delta})+\frac{1}{K}\sum_{m=1}^{K-1}(1-\frac{1}{K})^{K-1-m}\sum_{q=1}^{Q}\E(\langle (\one-\tilde{\x}^{(m)}_{q})\odot \g_{q}^{(m)},\vbf^{(m+1)}_{q}-\x^{*}_{\delta}\rangle)\\
    &+\frac{1}{K}\sum_{q=1}^{Q}\sum_{m=1}^{K-1}(1-\frac{1}{K})^{K-1-m}\E(\langle(\tilde{\vbf}^{(m+1)}_{q}-\tilde{\x}^{*}_{\delta})\odot(\one-\tilde{\x}^{(m)}_{q}),\nabla\bar{F}_{q,m-1}(\tilde{\x}^{m}_{q})-(1-\delta)\g_{q}^{(m)}\rangle)-\frac{L_{0}Qr^{2}(\C)}{2K}\\
    \end{aligned} 
\end{equation} 
First, for any $m\le K$, $\sum_{q=1}^{Q}\E(\langle (\one-\tilde{\x}^{(m)}_{q})\odot \g_{q}^{(m)},\x^{*}_{\delta}-\vbf^{(m+1)}_{q}\rangle)\le M_{0}\sqrt{Q}$. Next,
\begin{equation}
    \begin{aligned}
    &\frac{1}{K}\sum_{m=1}^{K-1}\E(\langle(\tilde{\vbf}^{(m+1)}_{q}-\tilde{\x}^{*}_{\delta})\odot(\one-\tilde{\x}^{(m)}_{q}),\nabla\bar{F}_{q,m-1}(\tilde{\x}^{m}_{q})-(1-\delta)\one\odot\g_{q}^{(m)}\rangle)\\
    \ge & -\frac{1}{2K}\sum_{m=1}^{K-1}\E(\mathrm{diam}^{2}(\C)b+\|\nabla\bar{F}_{q,m-1}(\tilde{\x}^{m}_{q})-(1-\delta)\one\odot\g_{q}^{(m)}\|^{2}/b)\\
    \ge & -\frac{\mathrm{diam}(\C)}{2\delta K^{1/3}}-\frac{\mathrm{diam}(\C)N_{2}\delta}{2K^{1/3}},
    \end{aligned}
\end{equation} 
where the first inequality comes from the Cauchy inequality; the second follows from \cref{lemma:vr3} and $b=\frac{1}{\mathrm{diam}(\C)K^{1/3}\delta}$. Finally, due to $(1-\frac{1}{K})^{K-1}\ge\frac{1}{e}$, we have
\begin{equation}
    \sum_{q=1}^{Q}\frac{1}{e}\bar{F}_{q}(\x^{*}_{\delta})-\sum_{q=1}^{Q}\E(\bar{F}_{q}(\tilde{\x}^{(K)}_{q}))\le\frac{L_{0}Qr^{2}(\C)}{2K}+M_{0}\sqrt{Q}+\frac{\mathrm{diam}(\C)Q}{2\delta K^{1/3}}+\frac{\mathrm{diam}(\C)N_{2}\delta Q}{2K^{1/3}}.
\end{equation}
\end{proof}
Next, we prove \cref{thm:3}.
\begin{proof} 
Finally, according to \cref{lemma:3+} and \cref{lemma:final1}, we have
\begin{equation*}
    \begin{aligned}
    &\sum_{t=1}^{T}\frac{1}{e}f_{t}(\x^{*})-\sum_{t=1}^{T}\E(f_{t}(\y_{t}))\\
    \le & \frac{L_{0}r^{2}(\C)}{2}\frac{LQ}{K}+M_{0}L\sqrt{Q}+\frac{\mathrm{diam}(\C)LQ}{2\delta K^{1/3}}+\frac{\mathrm{diam}(\C)N_{2}\delta LQ}{2K^{1/3}}+2M_{1}KQ\\
    & +\left((\sqrt{n}+1)\frac{r(\C)}{r}+\sqrt{n}+2\right)TG\delta\\
    \le& C_{1}\frac{LQ}{K}+M_{0}L\sqrt{Q}+\frac{C_{2}LQ}{2\delta K^{1/3}}+\frac{C_{3}\delta LQ}{2K^{1/3}}+2M_{1}KQ+C_{4}T\delta,
	\end{aligned}	
	\end{equation*} where the first inequality follows from the $N_{2}\le\max\{3^{2/3}G^{2},8G^{2}+3(4.5L_{0}r(\C)+3G)^{2}/2\}+8\frac{n^{2}M_{1}^{2}}{\delta^{2}}$ and in the second inequality, we set $C_{1}=\frac{L_{0}r^{2}(\C)}{2}$, $C_{2}=(8n^{2}M_{1}^{2}+1)\mathrm{diam}(\C)$, $C_{3}=\max\{3^{2/3}G^{2},8G^{2}+3(4.5L_{0}r(\C)+3G)^{2}/2\}\mathrm{diam}(\C)$ and $C_{4}=((\sqrt{n}+1)\frac{r(\C)}{r}+\sqrt{n}+2)G$.
\end{proof}

\end{document}